\PassOptionsToPackage{table}{xcolor}
\documentclass{article}
\usepackage{referencepaper}

\usepackage{amsmath,amsfonts,bm}

\def\eqref#1{equation~\ref{#1}}

\def\1{\bm{1}}

\DeclareMathAlphabet{\mathsfit}{\encodingdefault}{\sfdefault}{m}{sl}
\SetMathAlphabet{\mathsfit}{bold}{\encodingdefault}{\sfdefault}{bx}{n}

\DeclareMathOperator*{\argmax}{arg\,max}

\usepackage{amssymb,amsthm}
\theoremstyle{definition}

\theoremstyle{plain}
\newtheorem{proposition}{Proposition}

\theoremstyle{remark}
\newtheorem{remark}{Remark}
\usepackage{booktabs,multirow,array,tabularx,threeparttable}
\usepackage{float,placeins,flafter}
\usepackage{algorithm}
\usepackage{algpseudocode}
\usepackage{graphicx}
\usepackage{pifont}
\usepackage{enumitem}
\newcommand{\stitle}[1]{\vspace{1mm} \noindent {\bf #1}}
\usepackage{xcolor}

\usepackage{hyperref}
\usepackage{url}
\usepackage{etoc}

\newcolumntype{Y}{>{\raggedright\arraybackslash}X}
\newcommand{\apptabsetup}{%
  \small
  \setlength{\tabcolsep}{5pt}%
  \renewcommand{\arraystretch}{1.12}%
}

\definecolor{AntBright}{HTML}{0066FF}
\definecolor{AntDeep}{HTML}{103CC8}
\definecolor{AntAqua}{HTML}{7DFFFF}
\definecolor{AntInk}{HTML}{000000}
\definecolor{AntSlate}{HTML}{526176}
\definecolor{AntRule}{HTML}{C9DAEF}
\definecolor{AntBlue}{HTML}{2F5F9F}
\definecolor{AntPale}{HTML}{EDF5FF}
\definecolor{AntPaleStrong}{HTML}{DCEBFA}
\definecolor{AntOracle}{HTML}{F1F3F5}

\newcommand{\viss}{ViSS}
\newcommand{\tvs}{TVS}

\newcommand{\cone}{\ding{182}}
\newcommand{\ctwo}{\ding{183}}
\newcommand{\RQ}[1]{\textbf{RQ#1}}

\newcommand{\sbrt}{StreamingBench}
\newcommand{\ovob}{OVO-Bench}

\newcommand{\best}[1]{\textbf{#1}}

\newcommand{\tabsection}[2]{\multicolumn{#1}{@{}l}{\itshape #2}\\}
\newcommand{\vissrow}{\rowcolor{AntPale}}
\newcommand{\tvsrow}{\rowcolor{AntPaleStrong}}

\newcommand{\SBInstruct}{80.52}
\newcommand{\SBBaselineOP}{84.74}
\newcommand{\SBBaselineCR}{71.88}
\newcommand{\SBBaselineCS}{90.85}
\newcommand{\SBBaselineATP}{89.44}
\newcommand{\SBBaselineEU}{76.10}
\newcommand{\SBBaselineTR}{89.10}
\newcommand{\SBBaselinePR}{81.48}
\newcommand{\SBBaselineSU}{78.05}
\newcommand{\SBBaselineACP}{79.04}
\newcommand{\SBBaselineCT}{41.97}
\newcommand{\OVOInstruct}{58.97}
\newcommand{\OVOInstructOCR}{95.30}
\newcommand{\OVOInstructACR}{80.73} 
\newcommand{\OVOInstructATR}{82.76}
\newcommand{\OVOInstructSTU}{73.60}
\newcommand{\OVOInstructFPD}{73.27}
\newcommand{\OVOInstructOJR}{82.61}
\newcommand{\OVOInstructRT}{81.38}
\newcommand{\OVOInstructEPM}{54.55}
\newcommand{\OVOInstructASI}{62.16}
\newcommand{\OVOInstructHLD}{37.63}
\newcommand{\OVOInstructBwd}{51.45}
\newcommand{\OVOInstructREC}{19.91}
\newcommand{\OVOInstructSSR}{66.93}
\newcommand{\OVOInstructCRR}{45.42}
\newcommand{\OVOInstructFwd}{44.09}

\newcommand{\SBStrictGeneric}{84.97} 
\newcommand{\OVOStrictGeneric}{63.48}
\newcommand{\StrictGenericHLD}{67.20} 
\newcommand{\StrictGenericEPM}{56.23} 
\newcommand{\StrictGenericASI}{60.81} 
\newcommand{\StrictGenericREC}{20.63}
\newcommand{\OVOStrictGenericMean}{63.44}

\newcommand{\OVOViSSMean}{65.16}

\newcommand{\OVOPairedSeedGap}{+1.72}
\newcommand{\SBStrictGenericMean}{84.88}
\newcommand{\SBViSSMean}{84.70}
\newcommand{\SBPairedSeedGap}{-0.17} 
\newcommand{\HLDStrictGenericMean}{67.20} 
\newcommand{\HLDViSSMean}{74.91} 
\newcommand{\HLDPairedSeedGap}{+7.71} 
\newcommand{\OVOPairedBootstrapLow}{+0.80}
\newcommand{\OVOPairedBootstrapHigh}{+2.60}
\newcommand{\HLDPairedBootstrapLow}{+4.50}
\newcommand{\HLDPairedBootstrapHigh}{+11.30}
\newcommand{\CTFull}{44.04}
\newcommand{\RECFull}{21.20}
\newcommand{\HLDFull}{75.27}
\newcommand{\SBNoSuff}{84.89}
\newcommand{\OVONoSuff}{63.52}
\newcommand{\CTNoSuff}{46.63}
\newcommand{\RECNoSuff}{20.63}
\newcommand{\HLDNoSuff}{70.97}
\newcommand{\SBNoRep}{84.65}
\newcommand{\OVONoRep}{65.70}

\newcommand{\SBSealed}{84.77}
\newcommand{\OVOSealed}{65.21}
\newcommand{\SBAgentConf}{86.33}
\newcommand{\OVOAgentConf}{69.77}

\newcommand{\SBOracleThree}{91.42}
\newcommand{\OVOOracleThree}{78.52}
\newcommand{\SBEventOnly}{85.69}
\newcommand{\OVOEventOnly}{67.67}
\newcommand{\SBMaxConfidence}{86.77}
\newcommand{\OVOMaxConfidence}{68.43}
\newcommand{\SBNoAcquisition}{86.29}
\newcommand{\OVONoAcquisition}{69.72}
\newcommand{\SBRandomView}{86.05}
\newcommand{\OVORandomView}{66.93}
\newcommand{\SBHistoryOnly}{86.45}
\newcommand{\OVOHistoryOnly}{68.84}
\newcommand{\SBStaticHistorySeventeen}{85.69} 

\newcommand{\SBTVSGain}{+1.56}
\newcommand{\OVOTVSGain}{+4.56}

\newcommand{\SBAlwaysHistory}{86.49}
\newcommand{\OVOAlwaysHistory}{69.51}
\newcommand{\SBNoCommit}{86.53}
\newcommand{\OVONoCommit}{69.75}

\newcommand{\EffHistoryCallsSB}{1.00}
\newcommand{\EffHistoryCallsOVO}{1.00}
\newcommand{\EffHistoryFramesSB}{29.97}
\newcommand{\EffHistoryFramesOVO}{30.54}
\newcommand{\EffStaticHistoryCallsSB}{1.00}
\newcommand{\EffStaticHistoryFramesSB}{17.00}
\newcommand{\EffAlwaysHistoryCallsSB}{2.00}
\newcommand{\EffAlwaysHistoryCallsOVO}{2.00}
\newcommand{\EffAlwaysHistoryFramesSB}{33.97}
\newcommand{\EffAlwaysHistoryFramesOVO}{34.54}
\newcommand{\EffRecentCallsSB}{1.00}
\newcommand{\EffRecentCallsOVO}{1.00}
\newcommand{\EffRecentFramesSB}{4.00}
\newcommand{\EffRecentFramesOVO}{4.00}
\newcommand{\EffEventOnlyCallsSB}{1.00}
\newcommand{\EffEventOnlyCallsOVO}{1.00}
\newcommand{\EffEventOnlyFramesSB}{19.51}
\newcommand{\EffEventOnlyFramesOVO}{19.79}
\newcommand{\EffMaxConfidenceCallsSB}{3.00}
\newcommand{\EffMaxConfidenceCallsOVO}{3.00}
\newcommand{\EffMaxConfidenceFramesSB}{53.48}
\newcommand{\EffMaxConfidenceFramesOVO}{54.33}
\newcommand{\EffNoAcquisitionCallsSB}{2.12}
\newcommand{\EffNoAcquisitionCallsOVO}{2.31}
\newcommand{\EffNoAcquisitionFramesSB}{27.26}
\newcommand{\EffNoAcquisitionFramesOVO}{33.31}
\newcommand{\EffRandomCallsSB}{1.00}
\newcommand{\EffRandomCallsOVO}{1.00}
\newcommand{\EffRandomFramesSB}{17.58}
\newcommand{\EffRandomFramesOVO}{18.10}
\newcommand{\EffNoCommitCallsSB}{1.60}
\newcommand{\EffNoCommitCallsOVO}{2.16}
\newcommand{\EffNoCommitFramesSB}{17.08}
\newcommand{\EffNoCommitFramesOVO}{30.33}
\newcommand{\EffTVSCallsSB}{1.60}
\newcommand{\EffTVSCallsOVO}{2.16}
\newcommand{\EffTVSFramesSB}{17.08}
\newcommand{\EffTVSFramesOVO}{30.33}
\newcommand{\EffOracleCallsSB}{3.00}
\newcommand{\EffOracleCallsOVO}{3.00}
\newcommand{\EffOracleFramesSB}{53.48}
\newcommand{\EffOracleFramesOVO}{54.33}
\newcommand{\SBFlowRecentKept}{2{,}049}

\newcommand{\SBFlowViewRepaired}{105}

\newcommand{\OVOFlowRecentKept}{1{,}688}

\newcommand{\OVOFlowViewRepaired}{318}

\newcommand{\SBRecentErrors}{380}
\newcommand{\OVORecentErrors}{1{,}206}
\newcommand{\SBRecentCorrect}{2{,}115}
\newcommand{\OVORecentCorrect}{1{,}829}
\newcommand{\SBRecoverable}{166}
\newcommand{\OVORecoverable}{447}
\newcommand{\SBNetCorrected}{39}
\newcommand{\OVONetCorrected}{177}
\newcommand{\SBOracleGapCases}{127}
\newcommand{\OVOOracleGapCases}{270}
\newcommand{\SBTraceNotAcquired}{32}
\newcommand{\OVOTraceNotAcquired}{53}
\newcommand{\SBTraceRejected}{10}
\newcommand{\OVOTraceRejected}{23}
\newcommand{\SBTraceOverridden}{85}
\newcommand{\OVOTraceOverridden}{194}

\newcommand{\OVOSealedRT}{82.20}
\newcommand{\OVOSealedBwd}{64.21}
\newcommand{\OVOSealedFwd}{49.20}

\newcommand{\OVOAgentConfRT}{81.27}
\newcommand{\OVOAgentConfBwd}{69.37}
\newcommand{\OVOAgentConfFwd}{58.67}

\newcommand{\TaskRECConf}{39.8}
\newcommand{\TaskOCRConf}{96.6}
\newcommand{\TaskACRConf}{80.7}
\newcommand{\TaskATRConf}{82.8}
\newcommand{\TaskSTUConf}{67.4}
\newcommand{\TaskFPDConf}{80.2}
\newcommand{\TaskOJRConf}{79.9}
\newcommand{\TaskEPMConf}{68.4}
\newcommand{\TaskASIConf}{68.2}
\newcommand{\TaskHLDConf}{71.5}
\newcommand{\TaskSSRConf}{74.1}
\newcommand{\TaskCRRConf}{62.1}

\newcommand{\SBSFTOP}{86.92}
\newcommand{\SBSFTCR}{84.38}
\newcommand{\SBSFTCS}{94.64}
\newcommand{\SBSFTATP}{93.07}
\newcommand{\SBSFTEU}{81.76}
\newcommand{\SBSFTTR}{92.21}
\newcommand{\SBSFTPR}{78.70}
\newcommand{\SBSFTSU}{81.71}
\newcommand{\SBSFTACP}{87.54}
\newcommand{\SBSFTCT}{44.04}
\newcommand{\SBConfOP}{91.28}
\newcommand{\SBConfCR}{83.59}
\newcommand{\SBConfCS}{93.38}
\newcommand{\SBConfATP}{92.08}
\newcommand{\SBConfEU}{83.02}
\newcommand{\SBConfTR}{95.64}
\newcommand{\SBConfPR}{87.04}
\newcommand{\SBConfSU}{80.49}
\newcommand{\SBConfACP}{84.99}
\newcommand{\SBConfCT}{54.92}

\newcommand{\TaskREC}{21.20}  
\newcommand{\TaskCRR}{51.67}  
\newcommand{\TaskEPM}{57.24}  
\newcommand{\TaskASI}{60.14}  
\newcommand{\TaskSTU}{72.47}  
\newcommand{\TaskSSR}{74.72}  
\newcommand{\TaskFPD}{75.25}  
\newcommand{\TaskHLD}{75.27}  
\newcommand{\TaskACR}{82.57}  
\newcommand{\TaskOJR}{83.15}  
\newcommand{\TaskATR}{84.48}  
\newcommand{\TaskOCR}{95.30}  

\newcommand{\TaskOCRBase}{\OVOInstructOCR}
\newcommand{\TaskATRBase}{\OVOInstructATR}
\newcommand{\TaskOJRBase}{\OVOInstructOJR}
\newcommand{\TaskACRBase}{\OVOInstructACR}
\newcommand{\TaskFPDBase}{\OVOInstructFPD}
\newcommand{\TaskSTUBase}{\OVOInstructSTU}
\newcommand{\TaskHLDBase}{\OVOInstructHLD}
\newcommand{\TaskASIBase}{\OVOInstructASI}
\newcommand{\TaskEPMBase}{\OVOInstructEPM}
\newcommand{\TaskSSRBase}{\OVOInstructSSR}
\newcommand{\TaskCRRBase}{\OVOInstructCRR}
\newcommand{\TaskRECBase}{\OVOInstructREC}
\newcommand{\TaskOCRTVS}{96.64}
\newcommand{\TaskATRTVS}{82.76}
\newcommand{\TaskOJRTVS}{79.89}
\newcommand{\TaskACRTVS}{80.73}
\newcommand{\TaskFPDTVS}{80.20}
\newcommand{\TaskSTUTVS}{67.42}
\newcommand{\TaskHLDTVS}{71.51}
\newcommand{\TaskASITVS}{68.24}
\newcommand{\TaskEPMTVS}{68.35}
\newcommand{\TaskSSRTVS}{74.09}
\newcommand{\TaskCRRTVS}{62.08}
\newcommand{\TaskRECTVS}{39.83}
\newcommand{\TaskOCROracle}{97.32}
\newcommand{\TaskATROracle}{87.93}
\newcommand{\TaskOJROracle}{87.50}
\newcommand{\TaskACROracle}{87.16}
\newcommand{\TaskFPDOracle}{83.17}
\newcommand{\TaskSTUOracle}{78.09}
\newcommand{\TaskHLDOracle}{82.26}
\newcommand{\TaskASIOracle}{79.73}
\newcommand{\TaskEPMOracle}{74.41}
\newcommand{\TaskSSROracle}{82.03}
\newcommand{\TaskCRROracle}{76.25}
\newcommand{\TaskRECOracle}{51.43}

\newcommand{\CaseBeanRecentMargin}{0.834}
\newcommand{\CaseBeanEventMargin}{0.991}

\newcommand{\CaseBeanDecisiveTime}{119.68}

\newcommand{\SBRepGain}{+4.25}        
\newcommand{\OVORepGain}{+6.24}       
\newcommand{\SBTotalGain}{+5.81}      
\newcommand{\OVOTotalGain}{+10.80}    
\newcommand{\SBStrictGenericDelta}{$-0.20$} 
\newcommand{\OVOStrictGenericDelta}{+1.73}  
\newcommand{\StrictGenericBaseDelta}{+4.51} 
\newcommand{\HLDStrictGenericDelta}{+8.06}  
\newcommand{\OVOSuffDrop}{1.69}       
\newcommand{\HLDSuffDrop}{4.30}       
\newcommand{\SuffRegressions}{124}
\newcommand{\SuffRepairs}{67}
\newcommand{\SuffP}{4.5{\times}10^{-5}} 
\newcommand{\SBOracleHeadroom}{5.09}  
\newcommand{\OVOOracleHeadroom}{8.75} 
\newcommand{\SBReplaceShare}{66.9\%}  
\newcommand{\OVOReplaceShare}{71.9\%} 
\newcommand{\RhoViewSB}{43.7\%}       
\newcommand{\RhoViewOVO}{37.1\%}      
\newcommand{\RhoRecoverSB}{63.3\%}    
\newcommand{\RhoRecoverOVO}{71.1\%}   
\newcommand{\RhoPreserveSB}{96.9\%}   
\newcommand{\RhoPreserveOVO}{92.3\%}  
\newcommand{\TrainMinutes}{98}        
\newcommand{\DeltaSSRRep}{+7.8}       
\newcommand{\DeltaCRRRep}{+6.3}       
\newcommand{\DeltaCRRep}{+12.5}       
\newcommand{\DeltaACPRep}{+8.5}       
\newcommand{\DeltaEURep}{+5.7}        
\newcommand{\DeltaPRAccess}{+8.3}     
\newcommand{\DeltaCTAccess}{+10.9}    

\title{Supervising the Visible State: Separating Representation from View Selection in Streaming Video}

\author{Yihao~Hu\textsuperscript{1,\ensuremath{\ast}},
Pan~Wang\textsuperscript{3,\ensuremath{\ast}},
Yitong~Lin\textsuperscript{2,\ensuremath{\ast}},
Zhihao~Wen\textsuperscript{4,\ensuremath{\dagger}},
Deyuan~Liu\textsuperscript{1}, Hang~Wang\textsuperscript{1},
Mingda~WAN\textsuperscript{1}, Yiwen~Zhao\textsuperscript{5},
Yutong~Li\textsuperscript{6}, Yuchen~Mao\textsuperscript{7},
Tao~Lin\textsuperscript{1,\ensuremath{\dagger}}}
\affiliations{%
\textsuperscript{1}Westlake University,\quad
\textsuperscript{2}ZJU,\quad
\textsuperscript{3}USTC,\quad
\textsuperscript{4}SMU,\quad
\textsuperscript{5}FDU,\quad
\textsuperscript{6}ECNU,\quad
\textsuperscript{7}SJTU\\[4pt]
\textsuperscript{\ensuremath{\ast}}Equal contribution,\quad
\textsuperscript{\ensuremath{\dagger}}Corresponding author}

\begin{document}
\etocdepthtag.toc{main}

\raggedbottom
\setlength{\parskip}{6pt}
\setlength{\textfloatsep}{6pt}
\setlength{\dbltextfloatsep}{6pt}
\setlength{\floatsep}{6pt}
\setlength{\dblfloatsep}{6pt}
\setlength{\intextsep}{6pt}
\setlength{\emergencystretch}{2em}

\maketitle
\begin{abstract}
An action seen once and the same action seen three times receive the same
action-recognition label, so that supervision does not require a model
to count. Answers shared across histories that differ
in event order, accumulation, or relations leave those distinctions
unconstrained; we call this \emph{state under-specification}. Visible-State
Supervision (\viss{}) addresses it by training on source-grounded probes of
event identity, order, accumulation, relation change, and answer sufficiency;
under the same recent-frame interface it improves OVO-Bench over a
strict-matched generic temporal supervision baseline while remaining comparable
on StreamingBench, with the largest gain in hallucination detection.
Test-Time Causal View Search (\tvs{}) then adapts evidence access while
keeping the learned model fixed, expanding from recent frames to sparse event and compact
history views on demand. The two interventions separate representation from
view selection: supervision particularly benefits hallucination detection and causal
reasoning, access particularly benefits event recall and forward reasoning. Together
they reach the best reported accuracy among the published systems compared here on
StreamingBench and OVO-Bench at a fraction of the processed frames of exhaustive selection, and strict
replay shows that most of the remaining gap to a three-view oracle is a
correct answer that was later replaced. \textbf{Streaming accuracy depends on
what supervision asks the model to keep, what evidence inference exposes, and
whether the model keeps a correct answer once it has one}.
\par\smallskip
{\small
\textbf{Code:} \url{https://github.com/yihaohu0118/ViSS}\\[2pt]
\textbf{Project Homepage:} \url{https://yihaohu0118.github.io/ViSS/}\\[2pt]
\textbf{Resources:} \href{https://huggingface.co/collections/mis2/viss-6adc8caf041ae0c501c9a11b}{ViSS model and code on Hugging Face}\par}
\end{abstract}

\section{Introduction}
\label{sec:intro}

Picture an assistant that watches through camera glasses. Midway
through a workout it is asked, ``How many push-ups was that?'' Its current
frame shows a person at the top of a push-up; whether that was the fifth or
the fifteenth is not in the frame. In the kitchen, ``Did the garlic go in
before the onions?'' and ``Did I turn the stove off?'' are questions of the
same kind, and if the stove was never in view, the honest answer is that it
cannot tell. This is the streaming setting: at time $T$, a model answers from
the observed prefix $V_{\leq T}$ under bounded context and computation
\citep{chen2024videollmonline,qian2025dispider}, without the complete
recording that offline models
\citep{li2024llavaonevision,chen2024internvl,cheng2024videollama2,zhang2024longva,bai2025qwen25vl}
and their benchmarks
\citep{mangalam2023egoschema,fu2024videomme,zhou2024mlvu} assume. What it can
answer depends on what it has retained about the history, not only on what it
currently sees.

The push-up question already shows what that retention must contain. Two
histories that end with the same movement, one after a single repetition and
one after three, receive the same action label; only a counting question
separates them. The same holds for the order of events and for the object relations
they leave behind. We call the temporal content that a bounded causal view
establishes its \emph{visible state}: which events occurred, in what order,
how often, which relations changed, and whether the evidence suffices to answer at all. \cone{} \emph{Answer-level supervision
under-specifies this state}: an answer shared across temporally different
histories leaves those distinctions unconstrained, so nothing in the target
requires the model to keep count, order, or relations. We call this
supervision-level ambiguity \emph{state under-specification}; temporal
benchmarks show that exactly these distinctions are where video LLMs remain
weak \citep{liu2024tempcompass,cai2024temporalbench,zhang2024vinoground,shangguan2025tomato}.
Related work on sequential state tracking makes intermediate progress explicit
by aligning recurrent predictions with operation-prefix states
\citep{wang2026par}, connecting the supervision target to the distinctions
needed during computation.

\begin{figure}[t]
\centering
\includegraphics[width=\textwidth]{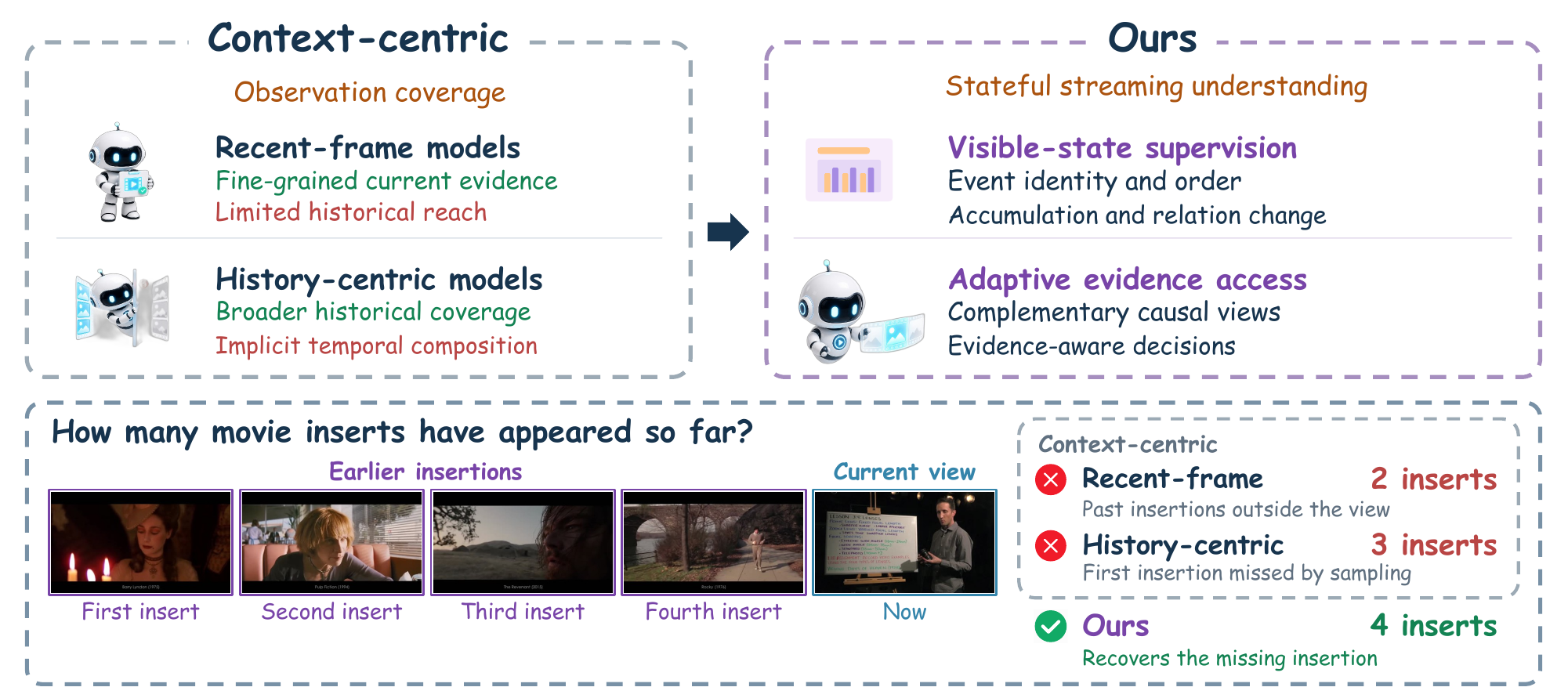}
\caption{\textbf{Context-centric streaming versus supervising the visible
state.} \textbf{Top:} recent-frame models see little history and
history-centric models leave temporal composition implicit; both inherit
answer-level supervision, whereas we supervise the visible state and acquire
causal views on demand. \textbf{Bottom:} counting movie inserts, the recent
view misses earlier ones and historical sampling skips the first; retrieving it restores the count. This example does not separate the two levers.}
\label{fig:teaser}
\end{figure}

Most streaming methods are \emph{context-centric}: they organize the growing
visual context, through token compression
\citep{li2024llamavid,song2024moviechat,shu2025videoxl,yao2025timechatonline},
persistent memories
\citep{he2024malmm,zhang2024flashvstream,wang2025videollamb,zhang2025videochatonline,zeng2025streamforest},
or event-based retrieval
\citep{selectstream2026,folio2026,curvestream2026,wen2026eventmemagent}, and
thereby decide which evidence stays accessible; our focus is how the supervision
formulation shapes what the model can infer from that evidence. SimpleStream~\citep{shen2026simplestream}
shows that an off-the-shelf VLM reading a few recent causal frames already matches far more elaborate systems. Yet \ctwo{} \emph{a representation, however
stateful, can ground only the observations its view exposes}: with four
recent frames, the history that recall and counting questions need is absent
(Figure~\ref{fig:teaser}). Representation and view selection are therefore
two separable levers, and we measure each with a controlled intervention on
one backbone, Qwen3-VL-8B with the latest four frames.

In this paper, we move both levers. \textbf{First}, due to \cone{}, we design
\textbf{Visible-State Supervision} (\viss{}), which trains on probes whose targets require the missing distinctions; the 7{,}677-example mixture is grounded in source annotations or frames and trains in one epoch. Under the same recent-4 interface it improves both
benchmarks, and a strict-matched generic temporal SFT control shows that the
visible-state formulation adds to the OVO-Bench gain from temporal post-training
(Section~\ref{sec:results-data}). Sufficiency is part of this state: removing examples whose answer is that the evidence does not suffice is associated with lower hallucination-detection accuracy.
\textbf{Second}, to address \ctwo{}, we introduce \textbf{Test-Time Causal
View Search} (\tvs{}), which freezes the \viss{} checkpoint and changes only
the view, cascading from the recent view to a sparse event view when the
recent answer is uncertain and to a compact history view when the two
disagree. With the checkpoint frozen it adds further gains on both benchmarks
at a third to a half of the frames of exhaustive three-view selection
(Section~\ref{sec:results-data}).

Together the two levers reach \SBAgentConf{} on StreamingBench and
\OVOAgentConf{} on OVO-Bench, \SBTotalGain{} and \OVOTotalGain{} points over
the same backbone (Figure~\ref{fig:results}); strict replay attributes most
of the remaining gap to a three-view oracle to a correct answer that was held
and then replaced.

Our contributions are:\nopagebreak[4]
\begin{itemize}[topsep=-1pt,itemsep=1.5pt,parsep=0pt,partopsep=0pt]
\item We identify state under-specification in answer-level supervision and
address it with \viss{}, 7{,}677 source-grounded probes of event identity,
order, accumulation, relation change, and sufficiency. At a matched backbone
and budget, it exceeds strict-matched generic temporal SFT by \OVOStrictGenericDelta{} points on
OVO-Bench and \HLDStrictGenericDelta{} points on hallucination detection.
\item We distinguish supervision and evidence access through matched-training
and fixed-checkpoint comparisons. Supervision mainly improves
hallucination detection and causal reasoning; access, through \tvs{}, mainly
improves event recall and forward reasoning at a fraction of the frames of
exhaustive selection.
\item Strict replay attributes \SBReplaceShare{} and \OVOReplaceShare{} of
the remaining oracle gap to replacing an already correct answer, identifying
commitment stability as the main selection limitation within this subset.
\end{itemize}

\section{Related Work}
\label{sec:related}

\stitle{Streaming post-training.}
Prior work adapts multimodal models to causal observation through causal
temporal attention \citep{yan2025streamformer}, large-scale multitask
instructions \citep{xia2025streamo}, supervised visual clues and response
states \citep{videochat32026}, streaming interaction formats
\citep{wang2025mmduet,liu2024streamchat,chen2025livecc}, reasoning that
evolves with the stream \citep{guan2026vst,thinkstream2026} or is reinforced
with rule-based rewards \citep{feng2025videor1}, and self-generated causal
trajectories \citep{wen2026evostreaming}. Counting benchmarks also probe
streaming state maintenance \citep{liu2026svcbench}. In audio-visual reasoning,
OmniVideo-R1 combines query-intensive grounding and modality-attentive fusion
within a reinforcement-learning framework \citep{chen2026omnivideor1}.
We complement these
directions by separating temporal post-training, visible-state formulation,
and evidence access. In this work, we target the unconstrained temporal
distinctions with source-grounded probes under a fixed recent-4 view; the
trade-off is that no architectural state is added, so the model can only
ground what its view exposes.

\stitle{View selection.}
Streaming systems bound what the model sees by compressing visual tokens or
KV caches \citep{li2024llamavid,song2024moviechat,shu2025videoxl,yao2025timechatonline},
maintaining external or recurrent memories
\citep{he2024malmm,qian2024videostreaming,zhang2024flashvstream,wang2025videollamb,zeng2025streamforest},
retrieving cached context \citep{di2025rekv}, organizing history through
event-aware or query-addressable structures
\citep{curvestream2026,selectstream2026,folio2026,wen2026eventmemagent},
selecting frames \citep{tang2025aks,yu2025framevoyager}, or acting as agents
that decide what to inspect
\citep{wang2024videoagent,wang2025videotree,eva2026,yang2025streamagent}.
AtlasVA organizes reusable visual skills through
spatial heatmaps, visual exemplars, and symbolic text skills
\citep{wang2026atlasva}.
SimpleStream \citep{shen2026simplestream} shows that a few recent frames
already form a strong reference and motivates accessing history on demand
(Figure~\ref{fig:teaser}). In this work, \tvs{} makes that direction concrete
as a fixed family of three causal views with calibrated acquisition and
commitment on the frozen \viss{} checkpoint; the trade-off is a
query-independent view family rather than learned retrieval, which lets its
effect be measured separately from supervision.

\section{Problem Formulation}
\label{sec:setup}

\subsection{Causal Streaming Prediction}

For episode $e$, let $V_e=((f_i,\tau_i))_{i=1}^{N_e}$ denote a timestamped
video. A question $q_e$ with target answer $y_e$ is issued at an exogenous
interruption time $T_e$, when the system has observed the causal prefix
$P_e(T_e)=\bigl((f_i,\tau_i):\tau_i\leq T_e\bigr)$. The relevant evidence
may be concentrated near $T_e$, distributed across earlier observations, or
insufficient at the time of the query. A causal view is a budgeted rendering
of the observed prefix, and the model computes
\begin{equation}
x_e=g\!\left(P_e(T_e)\right),\quad
I(x_e)\subseteq\{i:\tau_i\leq T_e\},\qquad
z_e=F_\theta(x_e),\quad
\hat y_e=G_\theta(z_e,q_e),
\label{eq:representation-readout}
\end{equation}
where $I(x_e)$ contains the source-frame indices represented in $x_e$, the
view constructor $g$ may select, resize, and organize observations while
preserving their temporal order, $z_e$ is the visual-temporal representation
induced by the view, and $G_\theta$ is the shared answer readout.

\subsection{State Representation and View Selection}

A useful streaming representation preserves the temporal distinctions that
support many questions about the same observed process. We call this content
the \emph{visible state} of a causal view and evaluate it behaviorally,
through the family of temporal predictions that the shared readout $G_\theta$
supports from $z_e$.

View selection determines which causal observations enter $z_e$. Let $\pi_0$
denote the fixed recent-4 policy. An adaptive policy $\pi$ may sequentially
acquire legal views $X_e^\pi=(x_{e,0},\ldots,x_{e,K_e})$ and select among the
resulting answer proposals, subject to
$\bigl|\bigcup_{k=0}^{K_e} I(x_{e,k})\bigr|\leq B_f$ unique source frames and
$K_e+1\leq B_c$ model calls; every acquired view draws exclusively from
$P_e(T_e)$.

\stitle{Controlled interventions.}
For benchmark score $S(\theta,\pi)$, we measure the two factors through
\begin{equation}
\begin{aligned}
\Delta_{\mathrm{rep}}&=S(\theta^\star,\pi_0)-S(\theta_0,\pi_0),\\
\Delta_{\mathrm{access}}&=S(\theta^\star,\pi_{\mathrm{TVS}})
                         -S(\theta^\star,\pi_0).
\end{aligned}
\label{eq:controlled-interventions}
\end{equation}
The first contrast measures the full post-training gain at fixed evaluation
evidence; the second holds every learned parameter fixed and changes only
the view policy; their sum is the staged gain from the public initialization
to the complete system. Strict replay of all three views on the frozen
checkpoint then localizes availability and selection failures within its
proposal set.

\stitle{Scope of claims.}
Visible state is a behavioral, supervision-level notion: we add no
architectural state and claim no unique internal state. The two contrasts
are staged, not independent, since a view's value depends on the checkpoint
reading it. All controlled results use one 8B backbone and two benchmarks;
\tvs{} thresholds were fixed once on a source-disjoint holdout.

\section{Method}
\label{sec:method}

\viss{} first trains shared parameters on temporally structured targets;
\tvs{} then freezes them and varies the evidence presented at inference
(Figure~\ref{fig:method-overview}). Evaluating the checkpoint under recent-4
access before enabling \tvs{} yields the two contrasts in
Eq.~\ref{eq:controlled-interventions}.

\begin{figure}[t]
\centering
\includegraphics[width=1.00\textwidth]{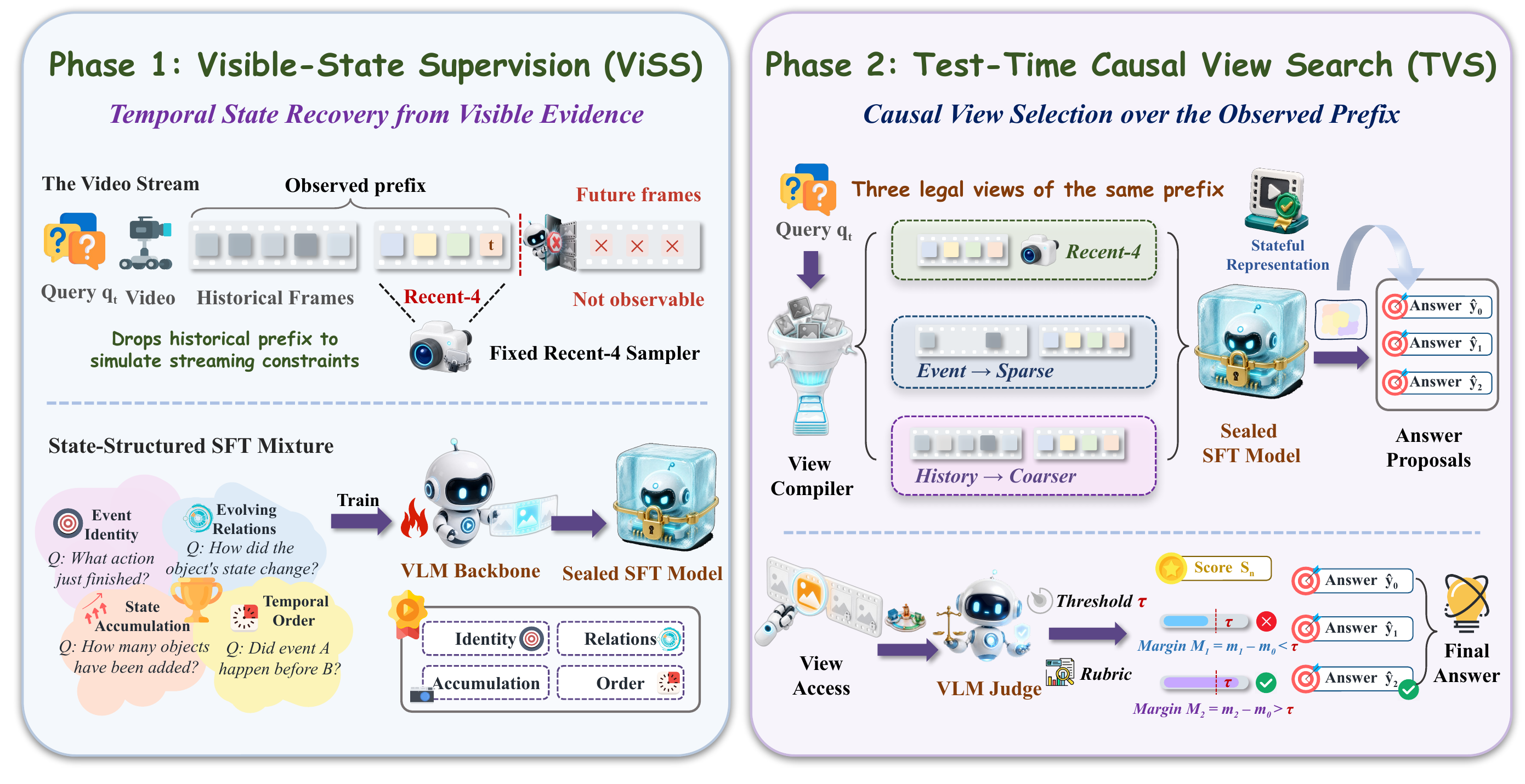}
\caption{\textbf{Overview.} \textbf{Phase 1, \viss{}
(Section~\ref{sec:method-viss}):} probes of event identity,
order, accumulation, and relation change (illustrated), together with sufficiency, train a shared
representation in one epoch; recent-4 is the dominant training input and the
evaluation interface (longer views: Appendix~\ref{app:configs}).
\textbf{Phase 2, \tvs{} (Section~\ref{sec:method-agent}):} the checkpoint is
frozen; the event view is acquired when the recent margin is low, the history
view when the two disagree, and a calibrated margin rule keeps one proposal.}
\label{fig:method-overview}
\end{figure}

\subsection{Visible-State Supervision}
\label{sec:method-viss}

\stitle{Motivation.}
By \cone{}, an answer-level target leaves multiplicity, order, relations, and
sufficiency unconstrained. The remedy is to ask for them directly: when targets are determined by the supplied view, a separating probe requires a loss-minimizing representation to keep the distinction (Proposition~\ref{prop:underspec}, Appendix~\ref{app:theory}), and asking for several jointly makes the requirements complementary
rather than task-specific.

\stitle{Probe families.}
Each probe pairs a causal training view $x_i$ with a question requesting one
temporal readout: identity probes ask which events occurred, order probes
their sequence, accumulation probes the quantities established by the prefix,
relation probes how observed states changed, and sufficiency probes whether
the supplied evidence supports a determinate answer, in the spirit of
unanswerable questions \citep{rajpurkar2018squad2}. Every family uses the
shared computation in Eq.~\ref{eq:representation-readout}, with a direct
option, binary value, or integer as the target. An identity target can be
satisfied without recovering an event's multiplicity or position; counting
and ordering targets make those distinctions necessary, and sufficiency
targets penalize commitment when the evidence is inadequate. Fixed-view
evaluation then tests whether the joint requirement transfers to temporal
decisions under bounded views.

\stitle{Source-grounded construction.}
We construct probes directly from source annotations whenever the required
event, timestamp, or relation is available; for capabilities absent from the
annotations, a teacher proposes candidate question--answer pairs from source
pixels, timestamps, and a family-specific instruction, as in synthetic video
instruction data \citep{maaz2024videochatgpt,chen2024sharegpt4video,zhang2025llavavideo},
behind a construction
firewall that excludes benchmark media, questions, answers, task metadata,
student predictions, and evaluation feedback. Every candidate is grounded in
the supplied causal prefix through source and timestamp checks, and a
\emph{semantic stability} check retains a candidate only when its answer maps
back to the same target under two option permutations $\pi_a$ and $\pi_b$,
$\pi_a^{-1}(\tilde y_i^{\pi_a})=\pi_b^{-1}(\tilde y_i^{\pi_b})=y_i$; the
auditor is not given $y_i$, and disagreeing records are removed rather than
relabeled. Source-membership, duplicate-question, and overlap checks complete
admission. These checks establish label stability and causal grounding, not
that every question needs history (Appendix~\ref{app:data-construction}).
The resulting 7{,}677-example mixture combines 3{,}594 source-aligned general
video examples, 2{,}172 repetition-counting probes built on RepCount
\citep{hu2022transrac}, 1{,}012 state probes over
longer prefixes, 116 sufficiency examples whose answer is that the evidence is
insufficient, and 783 teacher-proposed probes of backward state, action
sequence, and counting.

\stitle{Objective.}
Each admitted example pairs a causal view and probe with a direct target from
its answer space and supervision family. We optimize assistant-token-only
masked likelihood,
\begin{equation}
\mathcal L_{\mathrm{ViSS}}(\theta)
=
-\sum_i \alpha_i
\sum_{j\in\mathcal A_i}
\log p_\theta
\left(
y_{i,j}\mid x_i,q_i,y_{i,<j}
\right),
\label{eq:viss-objective}
\end{equation}
where $\mathcal A_i$ is the assistant-target mask and $\alpha_i$ the fixed
loss weight of example $i$; the target contains only the direct answer, with
source metadata, teacher confidence, rationales, and hidden environment state
excluded. A fresh run from the public initialization $\theta_0$ produces
$\theta^\star$ (configuration in Section~\ref{sec:exp-protocol} and
Appendix~\ref{app:configs}).

\subsection{Test-Time Causal View Search}
\label{sec:method-agent}

\stitle{Motivation.}
By \ctwo{}, even the \viss{} representation can ground only what its view
exposes, and recent-4 omits the history that recall and counting need.
Exposing the full history at every query is costly and often unnecessary, so
inference compute should go where it helps \citep{snell2025scaling}, and the
model's own confidence signals when more evidence is needed
\citep{kadavath2022know,kamath2020selective}. \tvs{} therefore expands the
view on demand with the checkpoint frozen.

\stitle{Causal views.}
Let $g_k$ denote a question-independent operator that renders the prefix as
$x_{e,k}=g_k(P_e(T_e))$; each view is processed by the same representation
and readout in Eq.~\ref{eq:representation-readout} with $\theta^\star$
frozen. The recent view $g_0$ keeps the latest four frames at full detail;
the event view $g_1$ adds up to 16 earlier frames ranked by visual change;
the compact-history view $g_2$ exposes a broader portion of the prefix at
reduced resolution while keeping high-resolution recent evidence. All three
share the causal endpoint $T_e$, preserve chronological order, and are
independent of the question, options, and task type; each call exposes at
most 32 frames, and a three-call path processes at most 56 frames including
repeats (Appendix~\ref{app:tvs}).

\stitle{Acquisition.}
\emph{What if the recent view is not enough?} Each call returns an answer
$\hat y_{e,k}$ and a confidence margin $m_{e,k}$, the difference between the
two largest full-vocabulary probabilities at the first answer position.
\tvs{} starts from the recent view and expands only when the current
prediction warrants it: a recent-view margin below $\tau_a$ activates the
event view, and disagreement between the recent and event proposals
activates compact history.

\stitle{Commitment.}
\emph{What if a later view is wrong?} Acquiring evidence and revising the
answer are separate decisions, and unlike voting over sampled answers
\citep{wang2023selfconsistency}, proposals are compared by margin across
views. A disagreeing proposal $k$ replaces the
retained proposal $j$ only when $m_{e,k}-m_{e,j}>\tau_c$; because the
calibrated $\tau_c$ is negative, an accepted proposal need not have a higher
raw margin than the retained answer. Both thresholds are calibrated once on a
source-video-disjoint holdout and fixed across benchmarks; exact rules, pseudocode, causality audits, and frame accounting are in Appendix~\ref{app:tvs}; Appendix~\ref{app:theory} analyzes the gate and the sign of $\tau_c$.

\section{Experiments}
\label{sec:experiments}

\begin{table}[t]
\caption{Full-split OVO-Bench accuracy by official task and group.
Overall is the official task-then-group macro; task names are listed in
Appendix~\ref{app:tasks}. Published rows keep their own protocols
(Appendix~\ref{app:protocols}) and serve as context; the Ours rows share one
backbone, scorer, and recent-4 starting view and carry the paper's claims.
Pale blue highlights the learned representation; darker blue denotes the
complete adaptive system. For \tvs{}, Frames is the per-call maximum.}
\label{tab:ovo-compare}
\centering
\scriptsize
\setlength{\tabcolsep}{1.3pt}
\renewcommand{\arraystretch}{1.0}
\resizebox{\linewidth}{!}{%
\begin{tabular}{@{}lc *{7}{c} *{4}{c} *{4}{c} c@{}}
\toprule
& & \multicolumn{7}{c}{Real-Time Visual Perception}
  & \multicolumn{4}{c}{Backward Tracing}
  & \multicolumn{4}{c}{Forward Active Responding} & \\
\cmidrule(lr){3-9}\cmidrule(lr){10-13}\cmidrule(lr){14-17}
Model & Frames
  & OCR & ACR & ATR & STU & FPD & OJR & Avg.
  & EPM & ASI & HLD & Avg.
  & REC & SSR & CRR & Avg.
  & Overall \\
\midrule
Human & --
  & 93.96 & 92.57 & 94.83 & 92.70 & 91.09 & 94.02 & 93.20
  & 92.59 & 93.02 & 91.37 & 92.33
  & 95.48 & 89.67 & 93.56 & 92.90
  & 92.81 \\
\midrule
\tabsection{18}{Proprietary MLLMs}
Gemini~1.5 Pro & 1\,fps
  & 85.91 & 66.97 & 79.31 & 58.43 & 63.37 & 61.96 & 69.32
  & 58.59 & 76.35 & 52.69 & 62.54
  & 35.53 & 74.24 & 61.67 & 57.15
  & 63.00 \\
GPT-4o & 64
  & 69.80 & 64.22 & 71.55 & 51.12 & 70.30 & 59.78 & 64.46
  & 57.91 & 75.68 & 48.66 & 60.75
  & 27.58 & 73.21 & 59.40 & 53.40
  & 59.54 \\
\midrule
\tabsection{18}{Open-source offline MLLMs}
Qwen2-VL-7B & 64
  & 60.40 & 50.46 & 56.03 & 47.19 & 66.34 & 55.43 & 55.98
  & 47.81 & 35.48 & 56.08 & 46.46
  & 31.66 & 65.82 & 48.75 & 48.74
  & 50.39 \\
InternVL-V2-8B & 64
  & 67.11 & 60.55 & 63.79 & 46.07 & 68.32 & 56.52 & 60.39
  & 48.15 & 57.43 & 24.73 & 43.44
  & 26.50 & 59.14 & 54.14 & 46.60
  & 50.15 \\
Qwen3-VL-8B & 64
  & 75.17 & 58.72 & 72.41 & 57.30 & 70.30 & 59.24 & 65.52
  & 56.57 & 69.59 & 12.37 & 46.18
  & 47.13 & 67.57 & 52.50 & 55.73
  & 55.81 \\
\midrule
\tabsection{18}{Open-source online MLLMs}
StreamForest-7B & 1\,fps
  & 68.46 & 53.21 & 71.55 & 47.75 & 65.35 & 60.87 & 61.20
  & 58.92 & 64.86 & 32.26 & 52.02
  & 32.81 & 70.59 & 57.08 & 53.49
  & 55.57 \\
\midrule
\tabsection{18}{Online video agents}
EventMemAgent & $\le$32
  & 75.84 & 69.72 & 73.28 & 55.62 & 67.33 & 67.93 & 68.29
  & 59.60 & 70.95 & 43.55 & 58.03
  & 33.67 & 72.02 & 62.08 & 55.92
  & 60.75 \\
\midrule
\tabsection{18}{Ours}
Qwen3-VL-8B-Instruct & 4
  & \OVOInstructOCR & \OVOInstructACR & \OVOInstructATR
  & \OVOInstructSTU & \OVOInstructFPD & \OVOInstructOJR
  & \OVOInstructRT
  & \OVOInstructEPM & \OVOInstructASI & \OVOInstructHLD
  & \OVOInstructBwd
  & \OVOInstructREC & \OVOInstructSSR & \OVOInstructCRR
  & \OVOInstructFwd
  & \OVOInstruct \\
\vissrow
\viss{} (recent-4) & 4
  & \TaskOCR & \TaskACR & \TaskATR & \TaskSTU & \TaskFPD & \TaskOJR & \OVOSealedRT
  & \TaskEPM & \TaskASI & \TaskHLD & \OVOSealedBwd
  & \TaskREC & \TaskSSR & \TaskCRR & \OVOSealedFwd
  & \OVOSealed \\
\tvsrow
\textbf{\viss{} + \tvs{}} & $\boldsymbol{\leq 32}$
  & \best{\TaskOCRConf} & \best{\TaskACRConf} & \best{\TaskATRConf}
  & \best{\TaskSTUConf} & \best{\TaskFPDConf} & \best{\TaskOJRConf}
  & \best{\OVOAgentConfRT}
  & \best{\TaskEPMConf} & \best{\TaskASIConf} & \best{\TaskHLDConf}
  & \best{\OVOAgentConfBwd}
  & \best{\TaskRECConf} & \best{\TaskSSRConf} & \best{\TaskCRRConf}
  & \best{\OVOAgentConfFwd}
  & \best{\OVOAgentConf} \\
\bottomrule
\end{tabular}}
\end{table}

We organize the analysis around four questions. \RQ{1}: does supervising visible-state distinctions help beyond generic temporal post-training at the same budget, and where is the additional gain concentrated? \RQ{2}: what does view selection add on a frozen checkpoint, and at what frame cost (both in Sections~\ref{sec:results-main} and~\ref{sec:results-data})?
\RQ{3}: do supervision and access improve the same tasks
(Section~\ref{sec:results-main})? \RQ{4}: when \tvs{} still fails, is the
correct answer unavailable or lost (Section~\ref{sec:results-controllers})?

\subsection{Experimental Setup}
\label{sec:exp-protocol}

We evaluate on \sbrt{}~\citep{lin2024streamingbench} (2{,}495 queries) and \ovob{} (3{,}035 queries) using
the full splits, the official scorers, and the same recent-4 starting view;
StreamingBench uses letter-exact accuracy, and OVO-Bench averages task
accuracies within its real-time, backward, and forward groups, then the three
groups. Published systems in Tables~\ref{tab:ovo-compare}
and~\ref{tab:sb-compare} keep their own protocols
(Appendix~\ref{app:protocols}); our claims rest on the controlled rows, which
share one backbone, scorer, and starting view. Within them, representation
comparisons vary the checkpoint under fixed recent-4 access, and access
comparisons freeze the \viss{} checkpoint and vary only the view policy; no
policy uses benchmark task labels or task-specific routing. The backbone is
Qwen3-VL-8B-Instruct~\citep{qwen3vl2026}: \viss{} is a fresh one-epoch SFT run with a frozen
vision tower, taking \TrainMinutes{} minutes on two H100 GPUs, and \tvs{}
starts from the same recent-4 input, makes at most three calls, and exposes at
most 32 causal frames per call, with thresholds fixed on a source-disjoint
holdout. Data construction and full configurations are in
Appendices~\ref{app:data-construction} and~\ref{app:configs}.

\begin{table}[t]
\caption{Full-split StreamingBench real-time accuracy by official task.
Column abbreviations follow the benchmark taxonomy. Published rows keep their
own protocols and serve as context; the Ours rows share one backbone, scorer,
and recent-4 starting view. Pale blue highlights the learned representation;
darker blue denotes the complete adaptive system. For \tvs{}, Frames is the
per-call maximum, not the union across calls.}
\label{tab:sb-compare}
\centering
\scriptsize
\setlength{\tabcolsep}{2.0pt}
\renewcommand{\arraystretch}{1.0}
\resizebox{\linewidth}{!}{%
\begin{tabular}{@{}lc *{10}{c} c@{}}
\toprule
Model & Frames
  & OP & CR & CS & ATP & EU & TR & PR & SU & ACP & CT
  & ALL \\
\midrule
Human & --
  & 89.47 & 92.00 & 93.60 & 91.47 & 95.65 & 92.52 & 88.00 & 88.75 & 89.74 & 91.30
  & 91.46 \\
\midrule
\tabsection{13}{Proprietary MLLMs}
Gemini~1.5 Pro & 1\,fps
  & 79.02 & 80.47 & 83.54 & 79.67 & 80.00 & 84.74 & 77.78 & 64.23 & 71.95 & 48.70
  & 75.69 \\
GPT-4o & 64
  & 77.11 & 80.47 & 83.91 & 76.47 & 70.19 & 83.80 & 66.67 & 62.19 & 69.12 & 49.22
  & 73.28 \\
Claude~3.5 Sonnet & 20
  & 80.49 & 77.34 & 82.02 & 81.73 & 72.33 & 75.39 & 61.11 & 61.79 & 69.32 & 43.09
  & 72.44 \\
\midrule
\tabsection{13}{Open-source offline MLLMs}
InternVL-V2-8B & 16
  & 68.12 & 60.94 & 69.40 & 77.12 & 67.70 & 62.93 & 59.26 & 53.25 & 54.96 & 56.48
  & 63.72 \\
Qwen2-VL-7B & 0.2--1\,fps
  & 75.20 & 77.45 & 73.19 & 82.81 & 68.32 & 71.03 & 72.22 & 61.19 & 61.47 & 46.11
  & 69.04 \\
Qwen3-VL-8B & 0.2--1\,fps
  & 73.57 & 71.09 & 76.34 & 80.39 & 65.22 & 74.77 & 78.70 & 67.89 & 61.47 & 47.67
  & 70.20 \\
\midrule
\tabsection{13}{Open-source online MLLMs}
StreamForest-7B & 1\,fps
  & 83.11 & 82.81 & 82.65 & 84.26 & 77.50 & 78.19 & 76.85 & 69.11 & 75.64 & 54.40
  & 77.26 \\
\midrule
\tabsection{13}{Online video agents}
EventMemAgent & $\le$32
  & 83.92 & 76.56 & 87.70 & 85.29 & 77.02 & 79.44 & 77.78 & 68.29 & 72.24 & 48.70
  & 77.00 \\
\midrule
\tabsection{13}{Ours}
Qwen3-VL-8B-Instruct & 4
  & \SBBaselineOP & \SBBaselineCR & \SBBaselineCS & \SBBaselineATP
  & \SBBaselineEU & \SBBaselineTR & \SBBaselinePR & \SBBaselineSU
  & \SBBaselineACP & \SBBaselineCT
  & \SBInstruct \\
\vissrow
\viss{} (recent-4) & 4
  & \SBSFTOP & \SBSFTCR & \SBSFTCS & \SBSFTATP & \SBSFTEU
  & \SBSFTTR & \SBSFTPR & \SBSFTSU & \SBSFTACP & \SBSFTCT
  & \SBSealed \\
\tvsrow
\textbf{\viss{} + \tvs{}} & $\boldsymbol{\leq 32}$
  & \best{\SBConfOP} & \best{\SBConfCR} & \best{\SBConfCS}
  & \best{\SBConfATP} & \best{\SBConfEU} & \best{\SBConfTR}
  & \best{\SBConfPR} & \best{\SBConfSU} & \best{\SBConfACP}
  & \best{\SBConfCT}
  & \best{\SBAgentConf} \\
\bottomrule
\end{tabular}}
\end{table}

\begin{figure}[!t]
\centering
\includegraphics[width=0.95\textwidth]{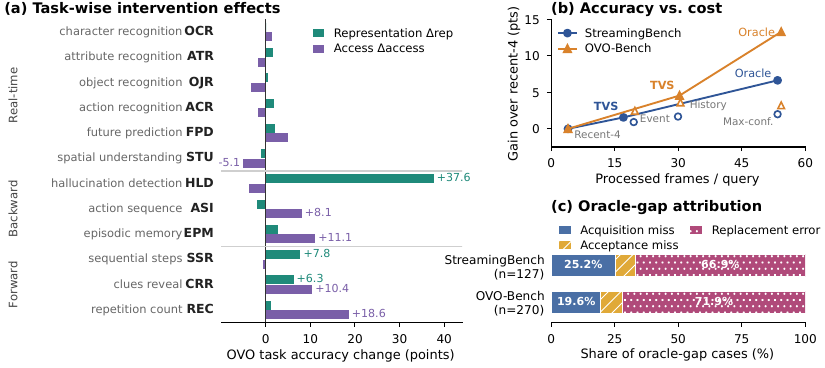}
\caption{What supervision and view selection each change, on one Qwen3-VL-8B
checkpoint. (a) OVO-Bench task accuracy change from \viss{} under fixed
recent-4 evidence ($\Delta_{\mathrm{rep}}$) and from \tvs{} with the
checkpoint frozen ($\Delta_{\mathrm{access}}$). (b) Gain over recent-4 versus
processed frames per query (filled: recent-4, \tvs{}, oracle; open: fixed
views, max-confidence). (c) Selection-error attribution conditioned on
oracle coverage: mostly a correct answer that was held and then replaced.}
\label{fig:results}
\end{figure}

\subsection{Comparison Experiments}
\label{sec:results}
\label{sec:results-main}

The recent-4 baseline is already strong: Qwen3-VL-8B with four recent frames
scores \SBInstruct{} and \OVOInstruct{}, above most published systems in
Tables~\ref{tab:ovo-compare} and~\ref{tab:sb-compare}, including ones that
maintain persistent memories or retrieve events \citep{shen2026simplestream},
so every gain below is measured against a competitive starting point on the
same backbone. Both interventions then add on both benchmarks: the complete \viss{} post-training recipe
raises the recent-4 scores by \SBRepGain{} and \OVORepGain{} after one epoch
on 7{,}677 examples, and \tvs{} adds \SBTVSGain{} and \OVOTVSGain{} with the
checkpoint frozen and at most three calls per query. \textbf{The complete
system reaches \SBAgentConf{} and \OVOAgentConf{}, \SBTotalGain{} and
\OVOTotalGain{} points above the same backbone. These are the highest reported scores
among the listed systems under their respective protocols, and the two
interventions improve different tasks}.

\stitle{OVO-Bench.}
Table~\ref{tab:ovo-compare} and
Figure~\ref{fig:results}(a) show that supervision and
access improve different tasks. Under the same recent-4 interface, the full post-training recipe
adds \OVORepGain{} points, and most of the Backward gain is hallucination
detection, a known weakness of vision-language models
\citep{guan2024hallusionbench}: HLD doubles from \OVOInstructHLD{} to \TaskHLD{}, whereas EPM moves
from \OVOInstructEPM{} to \TaskEPM{} and ASI from \OVOInstructASI{} to
\TaskASI{}. Against strict-matched generic supervision, HLD rises from
\StrictGenericHLD{} to \TaskHLD{}; transfer across the other temporal tasks is
smaller and heterogeneous. Knowing that the visible evidence does not answer the question is
part of the visible state, and the sufficiency-removal comparison in
Section~\ref{sec:results-data} shows a concentrated decline in hallucination detection. Forward
reasoning also improves under fixed views (SSR \DeltaSSRRep{}, CRR
\DeltaCRRRep{}). Freezing the checkpoint and enabling \tvs{} adds another
\OVOTVSGain{} points with the opposite profile: EPM rises to \TaskEPMTVS{}
and ASI to \TaskASITVS{}, REC nearly doubles from \TaskREC{} to
\TaskRECTVS{}, and CRR rises to \TaskCRRTVS{}, while HLD falls to \TaskHLDTVS{} and the real-time average slips from \OVOSealedRT{} to \OVOAgentConfRT{}. The slip is consistent with the imperfect preservation of correct recent answers measured in Section~\ref{sec:results-controllers}. In short, broader evidence buys recall and forward
reasoning at a small cost in hallucination detection and real-time perception
(\RQ{3}).

\stitle{StreamingBench.}
Table~\ref{tab:sb-compare} shows the same division. \viss{} adds
\SBRepGain{} points under recent-4 access, concentrated in causal reasoning
(CR \DeltaCRRep{}), action perception (ACP \DeltaACPRep{}), and event
understanding (EU \DeltaEURep{}), tasks that ask what the visible frames
establish; \tvs{} adds \SBTVSGain{} points, mainly in prospective reasoning
(PR \DeltaPRAccess{}) and counting (CT \DeltaCTAccess{}), tasks that depend
on frames the recent view no longer holds. The same split on a second benchmark shows that the two levers repair
different failures, and together they reach \SBAgentConf{}, the highest among the compared
systems in Table~\ref{tab:sb-compare}.
Strict-matched generic and \viss{} remain comparable on StreamingBench
(\SBStrictGenericMean{} versus \SBViSSMean{} across paired seeds).

\subsection{Diagnostic Analysis}
\label{sec:results-data}

\stitle{Supervision under strict matching.}
Table~\ref{tab:factorization-analysis}(a) asks whether the target structure
matters beyond additional temporal post-training. The control matches the
7{,}677 training instances, visual inputs, weighting, ordering, optimization,
initialization, and seed, comparing the question--target formulation. Generic
temporal post-training gains \StrictGenericBaseDelta{} OVO points over the base, and
\viss{} adds \OVOStrictGenericDelta{} while changing StreamingBench by
\SBStrictGenericDelta{}. Across three paired seeds the OVO gap is always
positive (mean \OVOPairedSeedGap{}); paired bootstrap of this mean gives a 95\% CI
of [\OVOPairedBootstrapLow{}, \OVOPairedBootstrapHigh{}]. The additional gain
is concentrated in HLD (Appendix~\ref{app:representation-control}; \RQ{1}).

Family removals assess component sensitivity. Removing the 116 sufficiency examples reduces OVO-Bench by \OVOSuffDrop{} and HLD by \HLDSuffDrop{}
points, with \SuffRegressions{} paired regressions against
\SuffRepairs{} repairs (McNemar $p=\SuffP{}$) and little change on
StreamingBench. Removal also changes training mass, so this is not an isolated semantic effect.
The repetition-state removal is retained only as a component-sensitivity
result (Appendix~\ref{app:representation-ablations}).

\begin{table}[!htbp]
\caption{Controlled analysis of state supervision and view selection.
(a) fixes recent-4 access; (b) fixes the \viss{} checkpoint.}
\label{tab:factorization-analysis}
\label{tab:access-ablation}
\centering
\small
\setlength{\tabcolsep}{4pt}
\renewcommand{\arraystretch}{1.12}

\begin{minipage}[t]{0.485\linewidth}
\vspace{0pt}
\centering
\textbf{(a) State supervision}
\par\vspace{4pt}
\begin{tabularx}{\linewidth}
{@{}>{\raggedright\arraybackslash}Xrr@{}}
\toprule
Supervision & SB & OVO \\
\midrule
Base
& \SBInstruct & \OVOInstruct \\
Strict-matched generic
& \SBStrictGeneric & \OVOStrictGeneric \\
\viss{} w/o sufficiency
& \SBNoSuff & \OVONoSuff \\
\rowcolor{AntPale}
\textbf{\viss{}}
& \SBSealed & \OVOSealed \\
\bottomrule
\end{tabularx}
\end{minipage}
\hfill
\begin{minipage}[t]{0.485\linewidth}
\vspace{0pt}
\centering
\textbf{(b) View selection}
\par\vspace{4pt}
\begin{tabularx}{\linewidth}
{@{}>{\raggedright\arraybackslash}Xrr@{}}
\toprule
View policy & SB & OVO \\
\midrule
Recent-4
& \SBSealed & \OVOSealed \\
Static history-17
& \SBStaticHistorySeventeen & -- \\
Full history
& \SBHistoryOnly & \OVOHistoryOnly \\
\rowcolor{AntPale}
\textbf{\tvs{}}
& \SBAgentConf & \OVOAgentConf \\
\bottomrule
\end{tabularx}
\end{minipage}

\end{table}

\stitle{Access at a fixed representation.}
Table~\ref{tab:access-ablation}(b) holds the \viss{} checkpoint fixed. Every
wider view beats recent-4, and a single history call already reaches
\SBHistoryOnly{} and \OVOHistoryOnly{}. What \tvs{} adds is adaptive evidence
access under a visual-frame budget. On StreamingBench it
nearly matches the history call (\SBAgentConf{} versus \SBHistoryOnly{}) with 43\%
fewer processed frames (\EffTVSFramesSB{} versus \EffHistoryFramesSB{}) and
beats static 17-frame sampling at a similar average frame budget by 0.64 points
(\SBAgentConf{} versus \SBStaticHistorySeventeen{}); on
OVO-Bench it exceeds both the history call and exhaustive max-confidence selection, which runs all three views and keeps the largest-margin answer (\OVOAgentConf{} versus \OVOHistoryOnly{} and
\OVOMaxConfidence{}) at \EffTVSFramesOVO{} frames against
\EffMaxConfidenceFramesOVO{} (Figure~\ref{fig:results}(b)). The acquisition gate is
what saves the frames: removing it raises processed frames to
\EffNoAcquisitionFramesSB{} and \EffNoAcquisitionFramesOVO{} at essentially
unchanged accuracy. The calibrated commit margin is an answer-selection heuristic with no consistent
aggregate gain: accepting every disagreeing proposal gives \SBNoCommit{} and
\OVONoCommit{} (Appendix~\ref{app:access-costs}; \RQ{2}).

\begin{figure}[!htbp]
\centering
\includegraphics[width=0.97\textwidth]{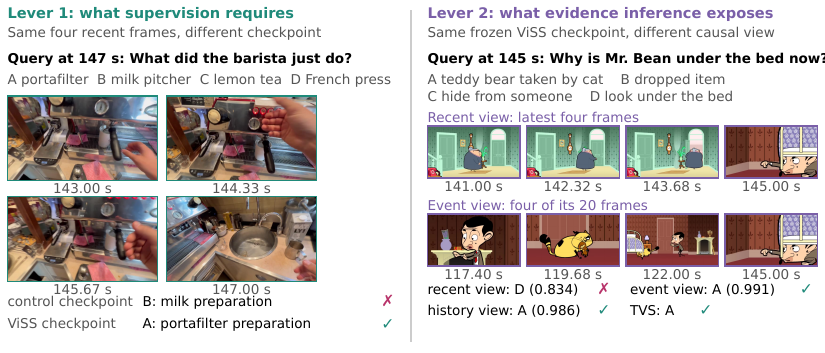}
\caption{One query per lever. Left: same four recent frames, two checkpoints
(control checkpoint: B, \viss{}: A). Right: one frozen \viss{} checkpoint, three
causal views; margins in parentheses, full trace in Figure~\ref{fig:qualitative-case-study}.}
\label{fig:case-levers}
\end{figure}

\subsection{Why Correct Alternatives Can Still Be Lost}
\label{sec:results-controllers}

Strict replay recomputes all three view proposals for every query on the frozen checkpoint, for diagnosis only, and separates two kinds of remaining error. A correct alternative-view proposal exists for
\RhoViewSB{} of recent-view errors on StreamingBench and \RhoViewOVO{} on
OVO-Bench; the rest are unresolved by every view (Appendix~\ref{app:all-view-unresolved-case} shows one whose ambiguity is what to count). When a correct proposal is available, \tvs{}
returns it in \RhoRecoverSB{} and \RhoRecoverOVO{} of cases, and it preserves
an initially correct recent answer in \RhoPreserveSB{} and \RhoPreserveOVO{}.

The three-view oracle, which uses the reference answer to pick a correct proposal whenever any view produces one, reaches \SBOracleThree{} and \OVOOracleThree{},
\SBOracleHeadroom{} and \OVOOracleHeadroom{} points above \tvs{}. Trace-level
attribution over the \SBOracleGapCases{} and \OVOOracleGapCases{} gap queries
assigns \SBReplaceShare{} and \OVOReplaceShare{} of them to replacement
errors, in which a correct proposal was retained and then replaced by a later
view (Figure~\ref{fig:results}(c)); acquisition misses
account for most of the remainder, and outright rejection of a correct
proposal is rare. Within these oracle-gap cases,
\textbf{retaining an available correct answer is the dominant selection limitation}.
This diagnosis is conditional on proposal coverage; it does not account for
all-view unresolved errors or establish a benefit of the current margin rule
(\RQ{4}; Appendices~\ref{app:error-analysis} and~\ref{app:audits}).

Related work studies selective correction through diagnosis-guided recovery
\citep{wang2026diagnosis} and the reliability of retaining intermediate
results across reasoning attempts \citep{liu2026intermediatereuse}.
Our replay analysis locates the corresponding selection problem within a
fixed visual proposal set: acquiring a correct alternative and preserving
it through subsequent revisions are distinct requirements.

\stitle{One query per lever.}
Figure~\ref{fig:case-levers} shows one query per lever. Left: on the same four recent frames, asked what the barista just did, the control checkpoint answers B (milk pitcher) and \viss{} answers A (lock the portafilter), the step the frames show: the difference lies in what the representation reads from one view (lever 1). Right: asked why Mr.\ Bean
is under the bed, the frozen checkpoint answers D from the recent view with
margin \CaseBeanRecentMargin{}, below the acquisition threshold; the event
view reaches back to the cat taking the teddy bear at
\CaseBeanDecisiveTime{}\,s and answers A with margin \CaseBeanEventMargin{},
the history view agrees, and \tvs{} commits to A (lever 2).

\section{Conclusion}
\label{sec:discussion}
\label{sec:conclusion}

Streaming accuracy depends on three separable factors: what the supervision
asks the model to preserve, what evidence inference exposes, and whether the
model keeps a correct answer once it has one. Visible-state supervision adds
to matched temporal post-training, particularly in hallucination detection;
causal view search addresses the second, above all event recall and forward reasoning. Within the gap to the current three-view oracle, replacing a retained correct answer is the dominant selection error, motivating more stable answer selection.

The controlled study uses one 8B backbone and two benchmarks; cross-scale
and cross-family transfer remain open. Gains vary across tasks: supervision
particularly benefits state fidelity, while adaptive access improves recall
and forward reasoning at a cost on some recent-state decisions.


\clearpage
\setcounter{topnumber}{2}
\setcounter{totalnumber}{3}
\subsection*{AI use statement}

Generative AI was used in two roles.  During data construction, a multimodal
teacher proposed candidate question--answer pairs and linguistic variants from
source video frames, timestamps, and generation instructions.  It had no
access to evaluation questions, reference answers, task metadata, or student
predictions; its outputs entered training only after deterministic semantic,
causal, source-membership, and overlap checks.  These checks do not certify
that every target requires historical evidence. AI assistants also supported
code editing, literature search, and manuscript revision.  The authors
designed the research, verified all code, citations, claims, and reported
numbers against source papers and versioned result records, and take responsibility
for the complete work.

\subsection*{Ethics statement}

This work uses existing research video datasets and public benchmarks; it
recruits no human participants and collects no new video.  Videos may contain
identifiable people or contextual information, so any release follows the
source licenses and does not redistribute benchmark media.  Benchmark examples
and labels are excluded from training, and benchmark identity and task type are
not available to the data compiler or \tvs{} policy.  Continuous video
understanding can nevertheless enable surveillance and inherit social or
demographic biases that we have not audited.  We therefore do not claim
suitability for surveillance or other high-stakes deployment.

\subsection*{Reproducibility statement}

Appendices~\ref{app:data-construction} and~\ref{app:configs} specify the data
pipeline, 7{,}677-example training mixture, model initialization, optimization,
and evaluation protocols.  Every retained SFT comparison is a fresh run from
the same public checkpoint; \tvs{} keeps the \viss{} checkpoint and its
holdout-calibrated thresholds fixed. Model weights, processors, source code,
and example configurations are available through the
\href{https://huggingface.co/collections/mis2/viss-6adc8caf041ae0c501c9a11b}{ViSS resource collection}.
Training annotations, raw media, and internal experiment records are not
included in this release. The view compiler excludes
question and answer content; timestamp and strict-replay checks are scoped
in Appendix~\ref{app:audits}. All reported values are defined once in
\texttt{numbers.tex} and traced to versioned result records.

\clearpage
\bibliography{references}
\bibliographystyle{plainnat}

\clearpage
\appendix
\etocdepthtag.toc{appendix}
\begingroup
  \setlength{\parskip}{0pt}
  \hypersetup{linktoc=all}
  \etocsettagdepth{main}{none}
  \etocsettagdepth{appendix}{subsection}
  \etocsetnexttocdepth{subsection}
  \etocsettocstyle{%
    \section*{Appendix Contents}%
    \pdfbookmark[0]{Appendix Contents}{appendix.contents}%
  }{}
  \tableofcontents
  \clearpage
\endgroup

\section{Causal Views and Adaptive View Selection}
\label{app:tvs}

This appendix specifies how complementary observations of one causal prefix
are constructed and sequentially consulted by \tvs{}.
For a query $q_e$ issued at time $t_e$, every view is compiled from
$V_{\leq t_e}$ independently of the query and answer content.
All policies use the same frozen \viss{} checkpoint $\theta^\star$, greedy
decoding, and the shared fixed configuration recorded below.

\subsection{Complementary Causal Views}

The recent view $x_{e,0}$ preserves local detail at the query frontier; the
event view $x_{e,1}$ adds sparse transitions from earlier in the prefix; and
the compact-history view $x_{e,2}$ trades historical spatial resolution for
greater temporal coverage. Table~\ref{tab:causal-view-family} defines these
observation operators. They share a causal endpoint, not an identical frame
grid: event and compact-history views are compiled separately and are not
strictly nested. Complementarity therefore concerns the observations they
provide, rather than a guarantee that every wider view preserves all earlier
source-frame IDs.

\begin{table}[!htbp]
\caption{Complementary causal views. Counts are per-call maxima; short
prefixes may supply fewer frames. Recent support uses the standard image
path, while compact historical images have maximum edge 224 pixels.}
\label{tab:causal-view-family}
\centering
\apptabsetup
\begin{tabularx}{\linewidth}{@{}lYYr@{}}
\toprule
View & Recent support & Historical support & Frames \\
\midrule
recent $x_{e,0}$ & 4 frames from the latest 4 seconds & none & 4 \\
event $x_{e,1}$ & the same recent 4 & up to 16 ranked frames, standard resolution & 20 \\
compact history $x_{e,2}$ & 8 frames from the latest 8 seconds
& up to 24 ranked frames, reduced resolution & 32 \\
\bottomrule
\end{tabularx}
\end{table}

For each historical view, we first reserve its recent-frame IDs. A uniform
grid of up to 64 positions spans the causal prefix; positions preceding the
earliest reserved frame supply historical candidates. We merge these with
the reserved recent frames, remove duplicate source IDs, and compute
$16\times16$ RGB descriptors $z_i$ in chronological order. Let
$\Delta z_i=z_i-z_{i-1}$ and let $R(\cdot)$ denote average-tie rank
normalization to $[0,1]$. Historical frames are ranked by the single fused
score
\begin{equation}
s_i = \frac{1}{15}\left[
4R\!\left(1-\cos(z_i,z_{i-1})\right)
+9R\!\left(1-\cos(\Delta z_i,\Delta z_{i+1})\right)
+2R(t_i)\right],
\label{eq:event-view-score}
\end{equation}
where $t_i$ is the source timestamp. The three terms capture appearance
novelty, curvature in visual evolution, and recency, respectively.
The ranking has no query-dependent term. Historical slots are filled by
ranked candidates, recent slots remain reserved, and selected frames are
presented chronologically. As an implementation convention, endpoint and
near-zero-norm curvature terms are zero; ties prefer the later timestamp and
then the smaller source-frame index. Recent windows use the Qwen sampler's
nominal 1\,FPS setting; endpoint rounding determines the exact frame grid.

\subsection{Sequential Evidence Selection}

For each acquired view $x_{e,k}$, the frozen checkpoint produces an answer
proposal $\hat y_k$ and a confidence margin $m_k$. The margin is the difference
between the two largest full-vocabulary probabilities at the first assistant
answer position.

\tvs{} begins from recent evidence. It consults the event view when the recent
proposal is uncertain, and consults compact history when the recent and event
views produce different answers. Acquisition changes the evidence presented
to the representation; commitment determines whether the resulting proposal
should replace the retained answer. Algorithm~\ref{alg:tvs} gives the complete
procedure.

\begin{algorithm}[H]
\caption{Sequential causal view acquisition and commitment}
\label{alg:tvs}
\small
\begin{algorithmic}[1]
\Require query $q_e$, causal prefix $V_{\leq t_e}$, compiler $\mathcal C$,
frozen checkpoint $\theta^\star$, thresholds $\tau_a,\tau_c$
\Ensure final answer $\hat y$
\State $(x_{e,0},x_{e,1},x_{e,2})\gets\mathcal C(V_{\leq t_e})$
\State $(\hat y_0,m_0)\gets\Call{DecodeMargin}{\theta^\star,q_e,x_{e,0}}$
\State $j\gets0$
\If{$m_0>\tau_a$}
    \State \Return $\hat y_j$
\EndIf
\State $(\hat y_1,m_1)\gets\Call{DecodeMargin}{\theta^\star,q_e,x_{e,1}}$
\If{$\hat y_1\neq\hat y_j$ \textbf{and} $m_1-m_j>\tau_c$}
    \State $j\gets1$
\EndIf
\If{$\hat y_1=\hat y_0$}
    \State \Return $\hat y_j$
\EndIf
\State $(\hat y_2,m_2)\gets\Call{DecodeMargin}{\theta^\star,q_e,x_{e,2}}$
\If{$\hat y_2\neq\hat y_j$ \textbf{and} $m_2-m_j>\tau_c$}
    \State $j\gets2$
\EndIf
\State \Return $\hat y_j$
\end{algorithmic}
\end{algorithm}

\subsection{Fixed Protocol and Validity}
\label{app:audits}

View construction and the two selection thresholds are fixed for both
benchmarks. Table~\ref{tab:tvs-fixed-config} records the configuration beyond
the view definitions in Table~\ref{tab:causal-view-family}.

\begin{table}[!htbp]
\caption{Fixed causal-view and search configuration used on both benchmarks.}
\label{tab:tvs-fixed-config}
\centering
\apptabsetup
\begin{tabularx}{\linewidth}{@{}lY@{}}
\toprule
Component & Setting \\
\midrule
Historical ranking
& $16\times16$ RGB descriptors; novelty, trajectory curvature, and recency
with weights $4{:}9{:}2$ \\
Candidate construction
& up to 64 prefix-grid positions, plus reserved recent IDs after filtering \\
Calibration
& 300 queries from 300 source-video-disjoint nonbenchmark videos \\
Acquisition threshold
& $\tau_a=0.9976896$ \\
Relative-margin tolerance
& $\tau_c=-0.5059$ \\
Inference budget
& at most 3 calls; at most 32 frames in any one call \\
Decoding
& frozen checkpoint, greedy decoding, first-answer full-vocabulary margin \\
\bottomrule
\end{tabularx}
\end{table}

The two thresholds are selected jointly on the 300-query disjoint calibration
set and then frozen for both benchmarks. We collect recent, event, and history
proposals and margins, use observed acquisition scores and disagreement-margin
differences as threshold candidates, and replay the complete controller for
each candidate pair. Selection maximizes calibration accuracy, breaking ties
by fewer processed frames and then fewer model calls.
The negative value of $\tau_c$ allows a disagreeing proposal to
replace the retained answer despite moderate cross-view variation in raw
margin; it should therefore be interpreted as a calibrated replacement
tolerance rather than a requirement for strictly higher confidence.

\stitle{Cost accounting.}
Compiler probes, unique model-exposed source IDs, and processed frames count
different operations. The 64-position prefix grid is not a cap on all
compiler inputs: recent IDs are added separately. The event-view records
contain up to 67 candidate descriptors. Likewise, 32 is a per-view limit,
not a bound on the union across calls. Executing all three views processes
at most $4+20+32=56$ frames, counting repeated exposure; short prefixes
reduce this total. Since recent-4 is included in the event view, the union
has the upper bound $20+32=52$ unique source IDs.
Table~\ref{tab:access-cost-full} reports actual mean calls
and processed frames, without treating these as wall-clock or token costs.

\stitle{Validity and audit scope.}
The compiler receives source frames, timestamps, and fixed budgets; questions,
answer options, predictions, and benchmark task labels do not enter its
ranking. Current complete event-view records cover 2{,}495 StreamingBench and
3{,}035 OVO-Bench queries with zero evaluation errors and no selected timestamp
beyond the query boundary, allowing $10^{-3}$ seconds for numerical rounding.
The separate strict-replay audit confirms complete query coverage, reproduction
of deployed \tvs{} answers, and exhaustive paired error accounting. These
checks establish timestamp compliance and replay consistency; they do not
independently establish that a selected view contains all evidence needed
to answer its query.

\section{Visible-State Supervision Data and Training}
\label{app:data-construction}

This section connects the visible-state targets to actual training records,
describes candidate verification, and records the training configuration.
The five temporal capabilities in the main text
describe what a question asks; the five mixture components below describe
where training records come from. These are not one-to-one categories.

\subsection{Supervision Units and Families}

A record contains a causal visual input, a question, its answer space, a
direct target, and source provenance. The temporal distinction is expressed
through the question--target pair, not through an extra latent-state label.
Table~\ref{tab:data-composition} gives the 7{,}677-row mixture. A component
may support several readouts: for example, a backward-state question can
require both event identity and temporal order.

\begin{table}[!htbp]
\caption{Supervision families in the sealed 7{,}677-example mixture.}
\label{tab:data-composition}
\centering
\apptabsetup
\begin{tabularx}{\linewidth}
{@{}>{\raggedright\arraybackslash}p{0.29\linewidth}Yr@{}}
\toprule
Family & Visible-state target & Rows \\
\midrule
source-aligned anchors
& general causal video understanding and capability retention
& 3{,}594 \\
repetition-counting probes
& accumulated event count over an ordered visual prefix
& 2{,}172 \\
prefix-32 state probes
& event state, temporal order, and accumulation over longer prefixes
& 1{,}012 \\
sufficiency anchors
& whether the available causal evidence supports an answer
& 116 \\
capability probes
& backward state retrieval, action sequence, and counting
& 783 \\
\midrule
Total & & 7{,}677 \\
\bottomrule
\end{tabularx}
\end{table}

The first three families form the 6{,}778-example core temporal mixture.
Sufficiency anchors teach the model to distinguish supported answers from
premature commitments. Capability probes extend the same visible-state
principle to backward memory, action sequence, and repetition-counting cases.

\stitle{Examples of the learning signal.}
Table~\ref{tab:supervision-records} shows two serialized records from the
training release. The first asks for completed repetitions, so merely naming
the visible action does not satisfy its target. The second asks for a gesture
before a later hand configuration, distinguishing temporal order from current
action recognition. These examples illustrate the intended readouts, not a
claim that every training question is impossible from a single frame.

\begin{table}[!htbp]
\caption{Actual training-record examples. Question text and targets are
preserved; common chat instructions and multiple-choice distractors are omitted.
Frame counts describe the serialized input, not the recent-4 evaluation view.}
\label{tab:supervision-records}
\centering
\apptabsetup
\begin{tabularx}{\linewidth}{@{}>{\raggedright\arraybackslash}p{0.19\linewidth}Y>{\raggedright\arraybackslash}p{0.17\linewidth}@{}}
\toprule
Source and input & Supervision question & Direct target \\
\midrule
RepCount annotations; 16 frames
& The frames are chronological causal observations. Count how many complete
situp repetitions have fully finished by this point. An incomplete motion
does not count. Reply with one integer only.
& \texttt{2} \\
\addlinespace[3pt]
Teacher capability probe; 12 frames
& What gesture did the woman make with her right hand before she began
speaking with her hand open and palm facing up?
& \texttt{D}: She pointed her index finger upwards. \\
\bottomrule
\end{tabularx}
\end{table}

\subsection{Source-Grounded Candidate Verification}

Direct probes are constructed from source annotations whenever the required
state and timestamp are available. Additional capability probes are proposed
by a teacher that receives source pixels, timestamps, and a family-specific
generation instruction. A construction firewall excludes benchmark media,
questions, labels, task metadata, student predictions, and evaluation
feedback from the proposal interface.

\stitle{Source binding and proposal checks.}
Generation requests specify the visible source frames and their ordering.
The builder joins a proposal to its request and source inventory by record
ID, checks the accepted decision and allowed confidence, rejects malformed
answers, and removes repeated normalized questions and benchmark-question
overlaps. Referenced images must belong to the supplied source inventory.
Evidence-group declarations and supporting facts, when included in the
proposal schema, are checked for internal consistency. Such declarations
remain teacher-generated evidence claims, not an independent visual proof.

\stitle{Answer stability without relabeling.}
The option-consensus filter presents the same source images and question
twice: once in the original option order and once in a deterministic
nonidentity permutation. The auditor does not receive the expected answer.
Both predictions must map back to the unchanged source target and satisfy
the configured confidence filter; disagreement removes the record rather
than replacing its target. For a teacher-proposed question, this checks the
proposed target, not an additional human annotation. The release check also
requires complete two-presentation coverage and no unresolved audit errors.

\stitle{What verification establishes.}
Source binding, causal input checks, overlap filtering, and option consistency
address distinct failure modes. They do not establish that an individual
question needs multiple frames or that its target is unambiguous to humans.
Repeated source pixels may also occur in distinct serialized observations;
media deduplication is not equivalent to removing every repeated exposure.
The retained capability set contains 783 probes: 452 backward-state,
236 action-sequence, and 95 repetition-counting probes.

We fix two empirical thresholds for \tvs{} (Appendix~\ref{app:tvs}).

\subsection{Optimization and Reproducibility}
\label{app:configs}

The sealed run starts from Qwen3-VL-8B-Instruct and uses the assistant-token
objective in Eq.~\ref{eq:viss-objective}. The language model and multimodal
projector are trainable; the vision tower is frozen.
Table~\ref{tab:training-reproducibility} groups the settings needed to reproduce
this run. Recent-4 is the fixed evaluation interface, not a restriction on
all training inputs: the release includes 4-, 12-, 16-, 24-, and 32-frame
records. Its serialized weights are also nonuniform: 6{,}281 rows have
$\alpha_i=1$ and 1{,}396 have $\alpha_i=8$.

\begin{table}[!htbp]
\caption{Reproducibility settings for the sealed \viss{} training run.}
\label{tab:training-reproducibility}
\centering
\apptabsetup
\begin{tabularx}{\linewidth}{@{}>{\raggedright\arraybackslash}p{0.28\linewidth}Y@{}}
\toprule
Setting & Value \\
\midrule
\multicolumn{2}{@{}l}{\textit{Model and input}} \\
Initialization & Qwen3-VL-8B-Instruct; fresh run from public weights \\
Trainable modules & language model and multimodal projector; vision tower frozen \\
Input limits & 8{,}192 tokens; 200{,}704 pixels per image \\
Loss & assistant targets only; serialized example weights \\
\addlinespace[3pt]
\multicolumn{2}{@{}l}{\textit{Data and optimization}} \\
Training data & 7{,}677 rows; mixed-length causal views \\
Ordering & 6{,}894-row parent stage, then 783 capability rows; stage shuffle seed 59 \\
Schedule & one epoch; AdamW; learning rate $2\times10^{-6}$; cosine decay \\
Regularization & warmup ratio 0.03; weight decay 0; gradient norm cap 1.0 \\
Batch & global 12; microbatch 1 per GPU; 6 accumulation steps \\
\addlinespace[3pt]
\multicolumn{2}{@{}l}{\textit{Execution}} \\
Precision and sharding & BF16, TF32 enabled; ZeRO-3 without CPU offload \\
Hardware and seed & 2$\times$NVIDIA H100 80GB; random seed 59 \\
Trainer runtime & 5{,}885.9\,s (98.1\,min), excluding data construction \\
\bottomrule
\end{tabularx}
\end{table}

Each record is rendered with the standard Qwen chat template. The user
message contains the ordered images, question, answer space, and direct-answer
instruction; the assistant target contains $y_i$. Source timestamps remain
in construction metadata unless explicitly included in the serialized prompt.
Rationales, teacher confidence, hidden
environment state, and benchmark metadata are excluded from the supervised
target. The training records, loss weights, and deterministic renderer
fully specify the training inputs.

\section{Controlled Experiments}
\label{app:controlled-experiments}
\label{app:supporting-ladders}

We organize the controlled experiments around the two factors isolated in the
main paper. Representation comparisons hold the recent-4 interface fixed and
vary the trained checkpoint. Access comparisons hold the frozen \viss{}
checkpoint fixed and vary only the causal view policy. For benchmark score
$S$, these comparisons follow the factorization in Section~\ref{sec:setup}, where
$\theta_0$ is the public initialization, $\theta^\star$ is the sealed \viss{}
checkpoint, and the recent-4 view realizes $\pi_0$.

\subsection{Representation under Fixed Recent Context}
\label{app:representation-control}

This comparison asks whether the \viss{} supervision formulation transfers
better than generic temporal post-training under matched training conditions. The paired
runs share the same 7{,}677 instances, visual inputs, weights, ordering,
optimization, initialization, seed, loss normalization, and recent-4
evaluation interface.

\stitle{Paired control construction.}
The pairing unit is a training-record slot with a fixed visual input and loss
weight. Write the two records as
$r_i^{G}=(x_i,q_i^{G},y_i^{G},\alpha_i)$ and
$r_i^{V}=(x_i,q_i^{V},y_i^{V},\alpha_i)$, where $x_i$ includes the ordered
source frames and their preprocessing. The row slot, visual exposure, and
weight are shared; the question and answer jointly express the generic or
visible-state formulation. This comparison evaluates the supervision
formulation under the shared training conditions, without separately
isolating question wording, label quality, or dependence on historical evidence.

\begin{table}[!htbp]
\caption{Strict-matched supervision control under recent-4 evaluation.}
\label{tab:representation-control}
\centering
\apptabsetup
\begin{tabularx}{\linewidth}{@{}>{\raggedright\arraybackslash}p{0.25\linewidth}YY@{}}
\toprule
Training attribute & Strict-matched generic & \viss{} \\
\midrule
Rows & 7{,}677 & 7{,}677 \\
Visual inputs and exposure & identical & identical \\
Weights, stages, and order & identical & identical \\
Optimization, initialization, seed & identical & identical \\
Question--target formulation & generic temporal QA & visible-state probes \\
StreamingBench & \SBStrictGeneric & \SBSealed \\
OVO-Bench & \OVOStrictGeneric & \OVOSealed \\
HLD / EPM / ASI / REC
& \StrictGenericHLD{} / \StrictGenericEPM{} / \StrictGenericASI{} / \StrictGenericREC{}
& \TaskHLD{} / \TaskEPM{} / \TaskASI{} / \TaskREC{} \\
\bottomrule
\end{tabularx}
\end{table}

\viss{} improves OVO-Bench by \OVOStrictGenericDelta{} points over the strict
control while remaining comparable on StreamingBench. HLD improves by
\HLDStrictGenericDelta{} points, corresponding to 15 additional correct answers
out of 186. All differences are computed before rounding the displayed scores.
Across paired seeds (Table~\ref{tab:paired-seeds}),
the mean OVO gap is \OVOPairedSeedGap{} and the mean HLD gap is
\HLDPairedSeedGap{}. Paired-bootstrap 95\% CIs for these three-seed mean effects are
[\OVOPairedBootstrapLow{}, \OVOPairedBootstrapHigh{}] and
[\HLDPairedBootstrapLow{}, \HLDPairedBootstrapHigh{}], respectively.
We use 10{,}000 source-video bootstrap replicates with a fixed random seed.
Each replicate applies the same video multiplicities to Generic and \viss{}
predictions across all three fixed training seeds, retaining every query from
each sampled video. We recompute the task--group--overall OVO macro within
each training seed and then average the three paired differences; HLD is
computed analogously within its task.

\begin{table}[!htbp]
\caption{Three paired training seeds, with Generic / \viss{} in each cell.
Seed rows give correct counts for StreamingBench ($N=2{,}495$) and HLD
($N=186$), and official macro scores for OVO. Mean scores are percentages;
mean paired differences are percentage points computed before rounding.}
\label{tab:paired-seeds}
\centering
\apptabsetup
\begin{tabular}{@{}lccc@{}}
\toprule
Seed pair & StreamingBench & OVO-Bench & HLD \\
\midrule
1 & 2{,}120 / 2{,}115 & 63.48 / 65.21 & 125 / 140 \\
2 & 2{,}117 / 2{,}113 & 63.21 / 64.90 & 124 / 138 \\
3 & 2{,}116 / 2{,}112 & 63.63 / 65.37 & 126 / 140 \\
\midrule
Mean score & \SBStrictGenericMean{} / \SBViSSMean{}
& \OVOStrictGenericMean{} / \OVOViSSMean{}
& \HLDStrictGenericMean{} / \HLDViSSMean{} \\
Mean paired difference & $\SBPairedSeedGap$ & $\OVOPairedSeedGap$ & $\HLDPairedSeedGap$ \\
\bottomrule
\end{tabular}
\end{table}

\subsection{Targeted Supervision Families}
\label{app:family-controls}
\label{app:representation-ablations}

Sufficiency is evaluated by removing its 116 examples from the complete
mixture under the same recent-4 interface. Removal changes visual exposure
and weighted training mass as well as supervision content; it is a component
sensitivity rather than an isolated estimate of sufficiency semantics.
Repetition-count supervision is
reported below only as an aggregate component sensitivity.

\begin{table}[t]
\caption{Sufficiency sensitivity under recent-4 evaluation.}
\label{tab:family-controls}
\label{tab:state-family-ablation-full}
\label{tab:representation-paired-full}
\centering
\begin{threeparttable}
\apptabsetup
\footnotesize

\textit{(a) Absolute benchmark and target-task scores.}

\vspace{2pt}
\begin{tabularx}{\linewidth}{@{}Yrrrrrr@{}}
\toprule
Configuration & Rows & SB & OVO & CT & REC & HLD \\
\midrule
full \viss{}
& 7{,}677 & \SBSealed & \OVOSealed
& \CTFull & \RECFull & \HLDFull \\
without sufficiency
& 7{,}561 & \SBNoSuff & \OVONoSuff
& \CTNoSuff & \RECNoSuff & \HLDNoSuff \\
\bottomrule
\end{tabularx}

\vspace{7pt}
\textit{(b) Paired query transitions for the corresponding contrasts.}

\vspace{2pt}
\begin{tabularx}{\linewidth}{@{}YYrrr@{}}
\toprule
Removed family & Benchmark & Repairs & Regressions & $p$ \\
\midrule
sufficiency & StreamingBench & 36 & 33 & 0.81 \\
sufficiency & OVO-Bench & 67 & 124 & $4.5{\times}10^{-5}$ \\
\bottomrule
\end{tabularx}

\begin{tablenotes}[flushleft]
\footnotesize
\item \emph{Paired accounting.} The reference is full \viss{}.
Repairs are queries corrected by the removed-sufficiency model; regressions are
queries lost by that model. The reported $p$ values use the two-sided exact
McNemar test.
They test paired query correctness, not differences in OVO's task-then-group
macro, and do not quantify variation across training seeds.
\end{tablenotes}
\end{threeparttable}
\end{table}

Removing sufficiency produces 124 regressions against 67 repairs on OVO-Bench
and concentrates its largest task-level decline on HLD. The corresponding
StreamingBench transition is nearly symmetric. Removing repetition-count
supervision changes aggregate performance only marginally: full \viss{}
scores \SBSealed{}/\OVOSealed{}, versus \SBNoRep{}/\OVONoRep{} without it,
so the overall gain is not driven by this component alone.

\subsection{Adaptive Access under Frozen Representation}
\label{app:access-costs}

All policies in this experiment share the frozen \viss{} checkpoint and the
same three query-independent causal views. Fixed-view controls isolate the
information supplied by each context. Selector controls test proposal choice.
Gate ablations change one decision rule at a time. Checkpoint identity is
matched; frame counts, resolutions, visual tokens, and calls are not.

\stitle{Fixed views and proposal selectors.}
Recent-only, event-only, and history-only each return one view's proposal.
Static-history-17 uniformly samples 17 causal frames on StreamingBench in one
call, approximating \tvs{}'s average processed-frame budget without adaptation.
Random-view selects one of the three by a fixed hash of the query record ID,
without consulting its content or predictions. Max-confidence evaluates all
three views and returns the largest margin, breaking ties in recent--event--
history order. The oracle also evaluates all three but uses the reference
answer to choose a correct proposal when one exists; it is analysis-only.

\stitle{Decision-rule ablations.}
Always-acquire-history makes two calls, recent then history, and applies the
same relative-margin criterion; it skips the event view. Without-acquisition-
gate always evaluates the event view, retaining both the disagreement trigger
for history and the commit criterion. Without-commit-margin preserves the
\tvs{} acquisition path but accepts every disagreeing acquired proposal.
Thus removing the acquisition gate does not force three calls, and removing
the commit margin does not change which views are requested.

\begin{table}[t]
\caption{Accuracy--cost trade-off under the frozen \viss{} representation.}
\label{tab:access-cost-full}
\centering
\begin{threeparttable}
\apptabsetup
\footnotesize
\begin{tabularx}{\linewidth}{@{}Yrrrrrr@{}}
\toprule
& \multicolumn{3}{c}{StreamingBench} & \multicolumn{3}{c}{OVO-Bench} \\
\cmidrule(lr){2-4}\cmidrule(l){5-7}
Access policy & Acc. & Calls/q & Frames/q & Macro & Calls/q & Frames/q \\
\midrule

\multicolumn{7}{@{}l}{\textit{Fixed one-view controls}} \\
recent only
& \SBSealed & \EffRecentCallsSB & \EffRecentFramesSB
& \OVOSealed & \EffRecentCallsOVO & \EffRecentFramesOVO \\
event only
& \SBEventOnly & \EffEventOnlyCallsSB & \EffEventOnlyFramesSB
& \OVOEventOnly & \EffEventOnlyCallsOVO & \EffEventOnlyFramesOVO \\
history only
& \SBHistoryOnly & \EffHistoryCallsSB & \EffHistoryFramesSB
& \OVOHistoryOnly & \EffHistoryCallsOVO & \EffHistoryFramesOVO \\
static history-17
& \SBStaticHistorySeventeen & \EffStaticHistoryCallsSB
& \EffStaticHistoryFramesSB
& -- & -- & -- \\
random view
& \SBRandomView & \EffRandomCallsSB & \EffRandomFramesSB
& \OVORandomView & \EffRandomCallsOVO & \EffRandomFramesOVO \\

\addlinespace[3pt]
\multicolumn{7}{@{}l}{\textit{Proposal selectors}} \\
max confidence
& \SBMaxConfidence & \EffMaxConfidenceCallsSB
& \EffMaxConfidenceFramesSB
& \OVOMaxConfidence & \EffMaxConfidenceCallsOVO
& \EffMaxConfidenceFramesOVO \\
\textbf{\tvs{}}
& \SBAgentConf & \EffTVSCallsSB & \EffTVSFramesSB
& \OVOAgentConf & \EffTVSCallsOVO & \EffTVSFramesOVO \\

\addlinespace[3pt]
\multicolumn{7}{@{}l}{\textit{Acquire--commit ablations}} \\
always acquire history
& \SBAlwaysHistory & \EffAlwaysHistoryCallsSB
& \EffAlwaysHistoryFramesSB
& \OVOAlwaysHistory & \EffAlwaysHistoryCallsOVO
& \EffAlwaysHistoryFramesOVO \\
without acquisition gate
& \SBNoAcquisition & \EffNoAcquisitionCallsSB
& \EffNoAcquisitionFramesSB
& \OVONoAcquisition & \EffNoAcquisitionCallsOVO
& \EffNoAcquisitionFramesOVO \\
without commit margin
& \SBNoCommit & \EffNoCommitCallsSB
& \EffNoCommitFramesSB
& \OVONoCommit & \EffNoCommitCallsOVO
& \EffNoCommitFramesOVO \\

\addlinespace[3pt]
\multicolumn{7}{@{}l}{\textit{Proposal-set upper bound}} \\
three-view oracle
& \SBOracleThree & \EffOracleCallsSB & \EffOracleFramesSB
& \OVOOracleThree & \EffOracleCallsOVO & \EffOracleFramesOVO \\
\bottomrule
\end{tabularx}
\begin{tablenotes}[flushleft]
\footnotesize
\item All entries use the same checkpoint and compiled views. Calls and
processed frames are averaged per query. Frames count input frame instances
summed across calls, including repeated frames, rather than unique source
frames. Compiler access is accounted for separately in Appendix~\ref{app:audits}.
\end{tablenotes}
\end{threeparttable}
\end{table}

\stitle{What the cost comparison supports.}
Against exhaustive max-confidence selection, \tvs{} uses fewer calls and
processed frames, improves OVO-Bench, and scores lower on StreamingBench.
Against a single history call, it uses fewer processed frames on
StreamingBench but almost the same number on OVO-Bench, while requiring more
calls on both. No universal latency advantage follows from these counts,
especially because historical image resolution varies across views.
At a comparable average StreamingBench frame budget (17.08 versus 17.00),
\tvs{} improves on static sampling by 0.64 points (86.33 versus 85.69),
using 1.60 rather than 1.00 calls per query.

The acquisition gate's clearest effect is cost: removing it raises processed
frames from 17.08 to 27.26 on StreamingBench and from 30.33 to 33.31 on
OVO-Bench, with accuracy changes of only $-0.04$ and $-0.05$ points.
Removing the commit margin changes accuracy from 86.33 to 86.53 and from
69.77 to 69.75, respectively. This ablation does not show a consistent
accuracy advantage for the margin rule; it also leaves frames and calls
unchanged. The calibrated answer-selection criterion defines the
deployed policy; the trace analysis in Appendix~\ref{app:error-analysis}
examines its remaining decision failures.

\subsection{Task-Wise Representation and Access Effects}
\label{app:tasks}

To localize the two interventions, we report each OVO task after the public
initialization, visible-state supervision, test-time view search, and oracle
proposal selection. The columns follow the same representation--access
progression as Section~\ref{sec:method}.

Figure~\ref{fig:results} in the main text summarizes
where each intervention helps, the resulting accuracy--cost trade-off, and
the attribution of the remaining oracle gap analyzed in
Appendix~\ref{app:error-analysis}; Table~\ref{tab:ovo-task-factorization}
gives the underlying task scores. In panel (b) of that figure, connected
filled markers show recent-4, \tvs{}, and the three-view oracle, and open
markers show fixed wider views and exhaustive max-confidence selection.

\begin{table}[!htbp]
\caption{Task-wise OVO accuracy across representation and access stages.
Task names follow \citet{li2025ovobench}; RT, Bwd, and Fwd denote real-time,
backward, and forward groups. All four stages use the official scoring rule.}
\label{tab:ovo-task-factorization}
\centering
\apptabsetup
\footnotesize
\begin{tabularx}{\linewidth}{@{}lYrrrr@{}}
\toprule
Group & Task & Base & \viss{} & $+\tvs{}$ & Oracle \\
\midrule
RT & Optical character recognition (OCR) & \TaskOCRBase & \TaskOCR & \TaskOCRTVS & \TaskOCROracle \\
RT & Attribute recognition (ATR) & \TaskATRBase & \TaskATR & \TaskATRTVS & \TaskATROracle \\
RT & Object recognition (OJR) & \TaskOJRBase & \TaskOJR & \TaskOJRTVS & \TaskOJROracle \\
RT & Action recognition (ACR) & \TaskACRBase & \TaskACR & \TaskACRTVS & \TaskACROracle \\
RT & Future prediction (FPD) & \TaskFPDBase & \TaskFPD & \TaskFPDTVS & \TaskFPDOracle \\
RT & Spatial understanding (STU) & \TaskSTUBase & \TaskSTU & \TaskSTUTVS & \TaskSTUOracle \\
\addlinespace[2pt]
Bwd & Hallucination detection (HLD) & \TaskHLDBase & \TaskHLD & \TaskHLDTVS & \TaskHLDOracle \\
Bwd & Action sequence identification (ASI) & \TaskASIBase & \TaskASI & \TaskASITVS & \TaskASIOracle \\
Bwd & Episodic memory (EPM) & \TaskEPMBase & \TaskEPM & \TaskEPMTVS & \TaskEPMOracle \\
\addlinespace[2pt]
Fwd & Sequential steps recognition (SSR) & \TaskSSRBase & \TaskSSR & \TaskSSRTVS & \TaskSSROracle \\
Fwd & Clues reveal responding (CRR) & \TaskCRRBase & \TaskCRR & \TaskCRRTVS & \TaskCRROracle \\
Fwd & Repetition event count (REC) & \TaskRECBase & \TaskREC & \TaskRECTVS & \TaskRECOracle \\
\midrule
\multicolumn{2}{l}{Official three-group macro}
& \OVOInstruct & \OVOSealed & \OVOAgentConf & \OVOOracleThree \\
\bottomrule
\end{tabularx}
\end{table}

The fixed-view Backward gain is primarily driven by HLD
(\OVOInstructHLD{} to \HLDFull{}), while EPM improves modestly and ASI
declines slightly. We therefore distinguish hallucination robustness from
event recall when interpreting this group, following
\citet{shen2026simplestream}. With the checkpoint frozen, \tvs{} improves EPM
and ASI but reduces HLD and the real-time average. Access also helps forward
reasoning, particularly REC and CRR. These effects support a task-dependent
account of supervision and access rather than uniform improvement across
temporal capabilities. REC reaches only \TaskRECOracle{} under oracle proposal selection.
This shows limits of the tested views, but does not establish REC as the
largest source of unresolved queries.

\stitle{Scope of external comparisons.}
\label{app:protocols}
Our controlled StreamingBench evaluations use letter-exact accuracy over all
2{,}495 real-time queries; OVO-Bench uses the official three-group macro over
real-time, backward, and forward tasks. Published systems retain their
original settings: EventMemAgent uses up to 32 frames; SimpleStream reports
RTVU and averages only real-time and backward OVO; VideoChat3 uses 4\,FPS with
up to 4{,}096 frames on StreamingBench and a 2\,FPS, 32-frame OVO window;
CurveStream and FOLIO report partial OVO task groups. We treat these rows as
protocol-specific context rather than strictly matched comparisons.
In particular, SimpleStream's two-group OVO average is not directly
comparable to our three-group macro. We retain the official task-then-group
scoring rule throughout and use task-wise diagnostics to interpret it, without
substituting a flat average over tasks.
The Gemini HLD value follows the current official OVO-Bench repository's
scoring example (52.69, with backward average 62.54).\footnote{\url{https://github.com/JoeLeelyf/OVO-Bench\#inference-and-score}}

\FloatBarrier
\section{Error Attribution: View Recoverability and Decision Stability}
\label{app:error-analysis}

We analyze errors through three concepts: whether a broader view makes an
error recoverable, whether \tvs{} selects an available correct proposal, and
whether the policy preserves a correct retained answer. For diagnosis only,
strict replay computes all three view proposals; it does not alter deployed
inference cost.

\subsection{Recoverability, Recovery, and Preservation}
\label{app:error-decomposition-full}

Let $\mathcal E_0$ denote the queries answered incorrectly from the recent
view. We further define
\begin{align}
\mathcal R
&=\{e\in\mathcal E_0:\exists k\in\{1,2\},\ \hat y_{e,k}=y_e\},\\
\mathcal G
&=\{e\in\mathcal R:\hat y_e^{\mathrm{TVS}}=y_e\},\\
\mathcal O
&=\{e:\hat y_{e,0}=y_e,\ \hat y_e^{\mathrm{TVS}}\neq y_e\}.
\end{align}
Here $\mathcal R$ contains recent-view errors for which the fixed view family
provides a correct alternative, $\mathcal G$ contains the cases recovered by
\tvs{}, and $\mathcal O$ contains correct recent answers lost during search.

We summarize the decomposition through three conditional diagnostics:
\begin{equation}
\rho_{\mathrm{view}}=\frac{|\mathcal R|}{|\mathcal E_0|},
\qquad
\rho_{\mathrm{recover}}=\frac{|\mathcal G|}{|\mathcal R|},
\qquad
\rho_{\mathrm{preserve}}=1-\frac{|\mathcal O|}{N-|\mathcal E_0|}.
\label{eq:error-diagnostics}
\end{equation}
They measure whether the view family contains the missing answer, whether
\tvs{} recovers an available answer, and whether it preserves an already
correct recent proposal.

\begin{table}[!htbp]
\caption{Conditional diagnostics of view availability and proposal selection.
Counts underlying each rate are shown in parentheses.}
\label{tab:error-decomposition-full}
\centering
\apptabsetup
\begin{tabularx}{\linewidth}
{@{}>{\raggedright\arraybackslash}p{0.27\linewidth}Yrr@{}}
\toprule
Diagnostic & Interpretation & StreamingBench & OVO-Bench \\
\midrule
View recoverability
& A broader causal view contains a correct proposal
& 43.7\% (\SBRecoverable/\SBRecentErrors)
& 37.1\% (\OVORecoverable/\OVORecentErrors) \\
Policy recovery
& \tvs{} returns the correct answer when one is available
& 63.3\% (\SBFlowViewRepaired/\SBRecoverable)
& 71.1\% (\OVOFlowViewRepaired/\OVORecoverable) \\
Recent-answer preservation
& \tvs{} retains an initially correct recent proposal
& 96.9\% (\SBFlowRecentKept/\SBRecentCorrect)
& 92.3\% (\OVOFlowRecentKept/\OVORecentCorrect) \\
\bottomrule
\end{tabularx}
\end{table}

The complementary all-view unresolved rates, 56.3\% on StreamingBench and
62.9\% on OVO-Bench, quantify the portion of recent-view errors outside the
current proposal set. Proposal selection alone cannot recover these cases.

Recent-answer preservation is 96.9\% on StreamingBench and 92.3\% on
OVO-Bench. Together, successful recoveries and lost recent answers yield net
gains of \SBNetCorrected{} and \OVONetCorrected{} correctly answered queries.
The StreamingBench count
corresponds to $2{,}154/2{,}495=86.33\%$; OVO counts remain diagnostic because
its official score is a task-then-group macro.

\subsection{Where Correct Proposals Are Lost}

For each oracle-gap query, we inspect the deployed execution trace. A query is
assigned to an \emph{acquisition miss} when no correct proposal is computed,
an \emph{acceptance miss} when a computed correct proposal is never retained,
and a \emph{replacement error} when a retained correct proposal is later
replaced. This ordered attribution produces mutually exclusive categories.

\begin{table}[!htbp]
\caption{Attribution of the remaining query-level gap to the three-view
oracle. Percentages are conditioned on oracle-gap cases.}
\label{tab:controller-failure-attribution}
\centering
\apptabsetup
\begin{tabularx}{\linewidth}
{@{}>{\raggedright\arraybackslash}p{0.25\linewidth}Yrr@{}}
\toprule
Failure mechanism & Trace-level criterion & StreamingBench & OVO-Bench \\
\midrule
Acquisition miss
& No correct proposal is computed along the deployed path
& \SBTraceNotAcquired{} (25.2\%) & \OVOTraceNotAcquired{} (19.6\%) \\
Acceptance miss
& A correct proposal is computed but never retained
& \SBTraceRejected{} (7.9\%) & \OVOTraceRejected{} (8.5\%) \\
Replacement error
& A retained correct proposal is subsequently replaced
& \SBTraceOverridden{} (66.9\%) & \OVOTraceOverridden{} (71.9\%) \\
\midrule
Total
& Errors containing at least one correct replayed proposal
& \SBOracleGapCases & \OVOOracleGapCases \\
\bottomrule
\end{tabularx}
\end{table}

Replacement errors dominate the remaining oracle gap, indicating that
stability after retaining a correct proposal is the main selection limitation
within this subset. Acquisition misses are the second-largest component;
direct rejection of an acquired correct proposal is less common.
Replacement counts include both initially correct recent answers and correct
intermediate event answers subsequently lost at the history stage. They are
therefore larger than the recent-answer regressions in
Table~\ref{tab:error-decomposition-full}.

\subsection{Failures Beyond the Current Proposal Set}

For 214 StreamingBench queries and 759 OVO-Bench queries, none of the three
views produces the reference answer. These cases constitute the majority of
recent-view errors on both benchmarks and establish that proposal selection
alone cannot close the remaining gap.

An unresolved query may reflect a decisive frame omitted by the compiler,
failed interpretation despite adequate evidence, or visual ambiguity.
Separating these causes requires dense-prefix evaluation or an independent
evidence-coverage audit. We therefore use ``all-view unresolved'' rather than
``representation error''.

\clearpage
\onecolumn
\section{Qualitative Case Studies}
\label{app:qualitative-case-study}

\subsection{A Correct Answer Acquired from History}

Figure~\ref{fig:qualitative-case-study} instantiates the
representation--access distinction on one causal query, showing the exact
frames, answer proposals, confidence margins, and sequential commitment
decisions.

\begin{figure}[H]
\centering
\includegraphics[height=0.75\textheight]{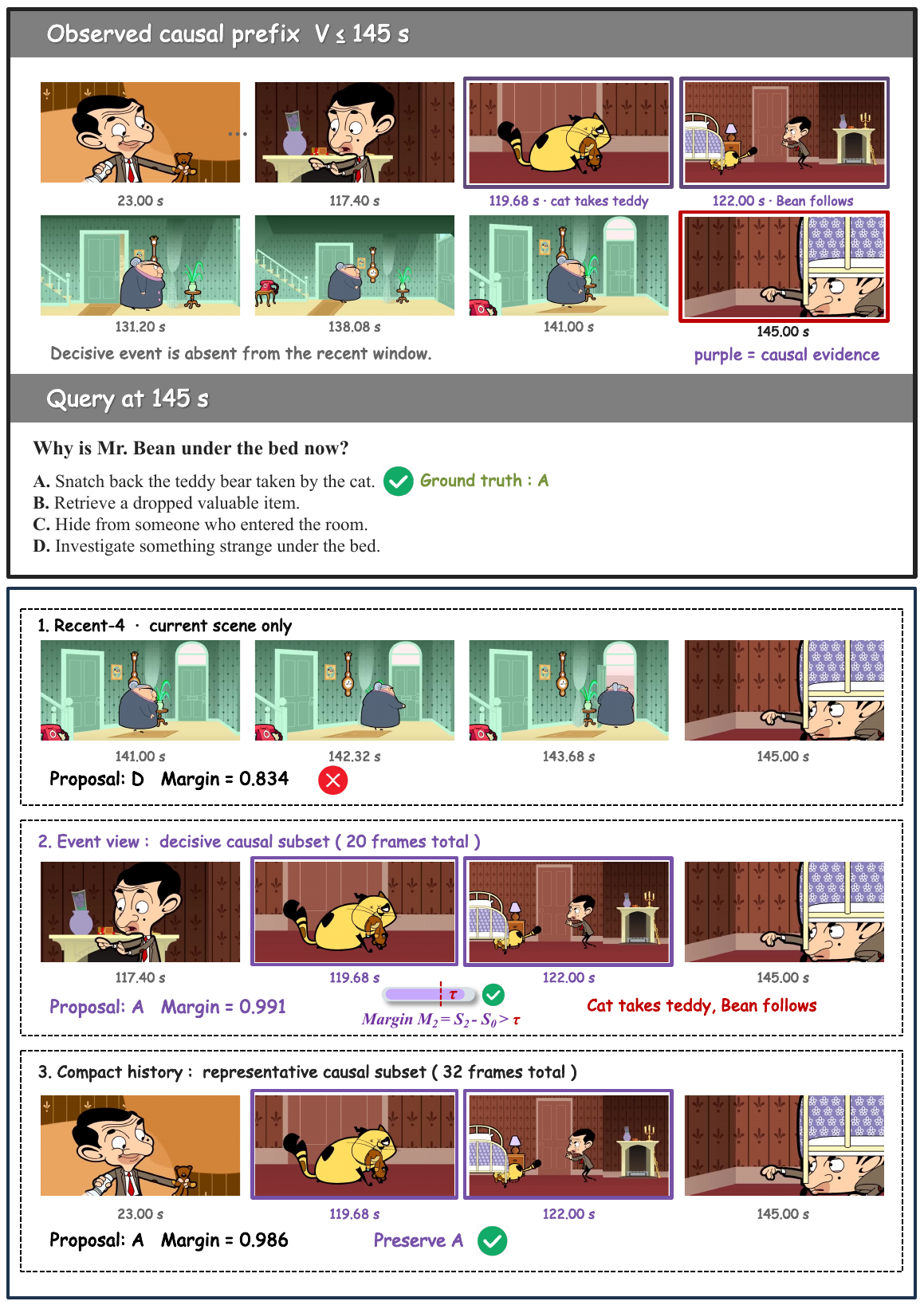}
\vspace{-0.45em}
\caption{Qualitative case study of causal view acquisition and proposal
commitment. Recent-4 omits the event in which the cat takes the teddy bear and
proposes D. The event view exposes the decisive evidence and proposes A; the
compact-history view independently supports A, which the acquire--commit
policy preserves.}
\label{fig:qualitative-case-study}
\end{figure}

\clearpage
\subsection{A Correct Answer Lost during Revision}
\label{app:replacement-case}

A complementary failure appears in StreamingBench record
\texttt{streaming\_57\_0}, from \texttt{sample\_137\_real} at 12 seconds.
The question and options are:
\begin{quote}
\small
\emph{Which of the following best describes the initial actions shown in
the clip?}

A. A hand is painting directly on a canvas.\\
B. A plastic sheet is covered with paint stains, then a name appears.\\
C. A blank canvas is being prepared for painting.\\
D. Paint is being dispensed onto a plastic sheet covering a canvas.
\end{quote}
The reference is D. Table~\ref{tab:replacement-case} reconstructs the
deployed decisions from the frozen-checkpoint view records and strict replay,
without generating new proposals or using the reference answer in the policy.

\begin{table}[!htbp]
\caption{A recorded replacement error. The recent answer is correct, but
both subsequent proposals pass the calibrated relative-margin criterion.
Margins are rounded for display.}
\label{tab:replacement-case}
\centering
\apptabsetup
\begin{tabularx}{\linewidth}{@{}Yrrrr@{}}
\toprule
Acquired view & Frames & Proposal & Margin & Retained \\
\midrule
recent & 4 & D & 0.5084 & D \\
event & 20 & B & 0.1228 & B \\
compact history & 29 & C & 0.2416 & C \\
\bottomrule
\end{tabularx}
\end{table}

\stitle{Acquisition and commitment have different effects.}
The recent margin triggers event acquisition because it is below
$\tau_a=0.9976896$. The event proposal has a lower raw margin, but its change
relative to the retained answer is approximately $-0.3856$, which exceeds
$\tau_c=-0.5059$; B therefore replaces the correct D. Recent--event
disagreement triggers history independently of this commitment. Its C
proposal then replaces B with a relative margin change of $+0.1188$.
The final answer is wrong even though a correct proposal was available
and initially retained.

\stitle{What the trace establishes.}
This is a replacement error in the sense of
Appendix~\ref{app:error-analysis}, not a case where the view family lacks a
correct proposal. It illustrates why the calibrated rule cannot be described
as accepting only higher-confidence answers. It does not establish that a
different threshold would improve aggregate performance, nor does it identify
which image caused the proposal to change.

All recorded frame timestamps lie at or before 12 seconds. The recent view
uses 8.00, 9.33, 10.67, and 12.00 seconds; the broader views also include
earlier observations. This path processes 53 frame entries but exposes 39
unique source IDs, illustrating the distinction between per-view capacity,
cross-call exposure, and repeated processing. Record IDs, exact margins,
and frame-index lists for this example accompany the manuscript source.

\clearpage
\subsection{A Correct Historical Proposal Not Acquired}
\label{app:acquisition-miss-case}

At 6 seconds, StreamingBench record \texttt{streaming\_461\_0} has reference
B (center), but recent-4 proposes A (left). Its margin, 0.999630, exceeds
$\tau_a\approx0.9977$, so \tvs{} stops after one call and four frames.
The correct event and history proposals in
Figure~\ref{fig:case-untriggered-history} are replay-only, not answers visited
by the deployed policy. Their earlier observations retain a left neighbour
absent from recent-4. This is an acquisition miss, not evidence that changing
the threshold would improve aggregate accuracy.

\begin{figure}[H]
\centering
\includegraphics[height=0.72\textheight,width=\textwidth,keepaspectratio]{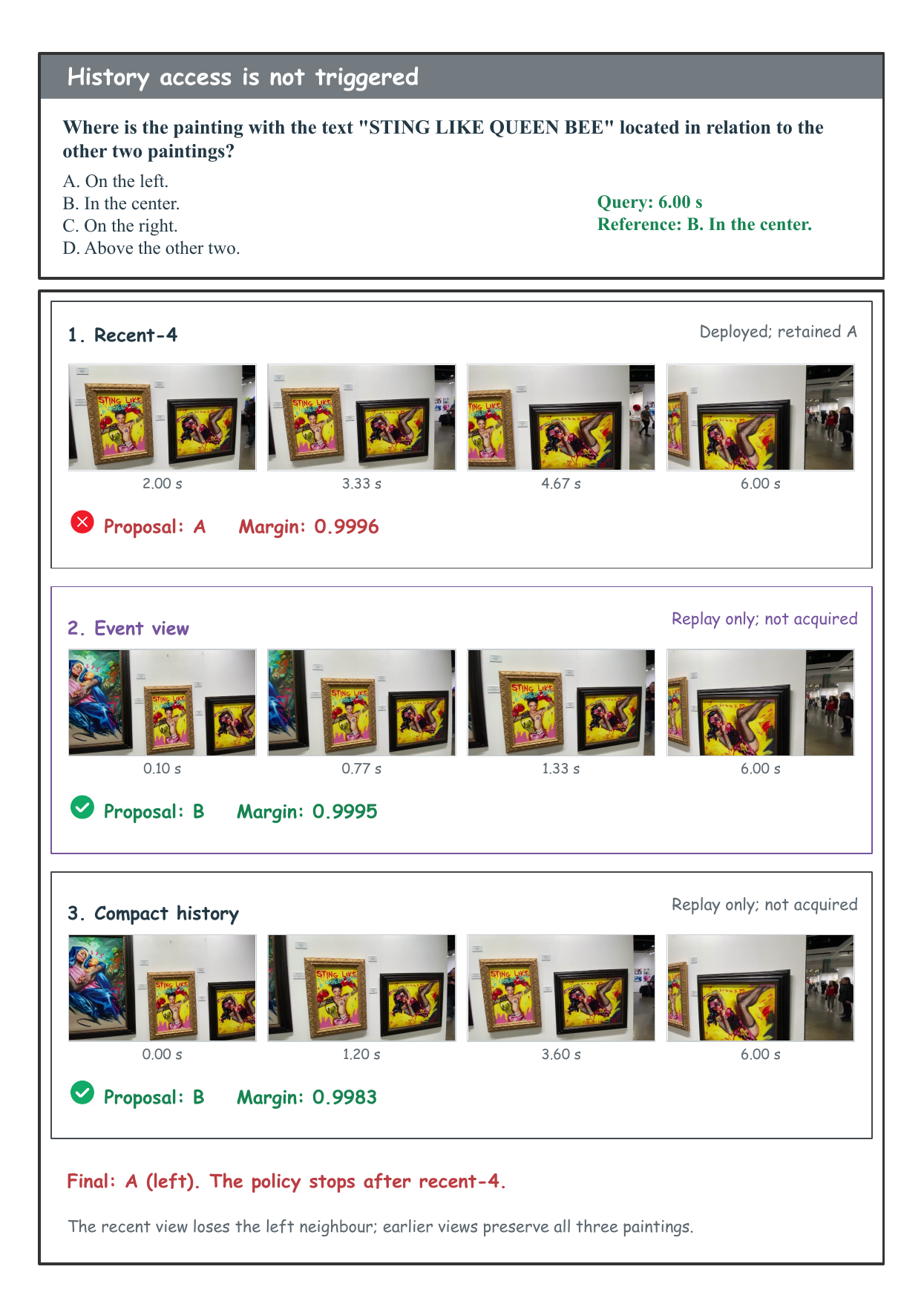}
\caption{\textbf{Corrective historical evidence is not accessed.}
Four actual frames per view are displayed; complete inputs contain 4 recent,
20 event, and 6 history frames. Earlier frames show all three paintings,
while recent-4 loses the left neighbour. All proposals use the same frozen
\viss{} checkpoint; only recent-4 is executed by \tvs{}.}
\label{fig:case-untriggered-history}
\end{figure}

\clearpage
\subsection{All-View Unresolved with Counting-Scope Ambiguity}
\label{app:all-view-unresolved-case}

At 10 seconds, record \texttt{streaming\_38\_1} has reference B (three), but
all three views propose C (four). Recent and event agree, so \tvs{} stops
after two calls and 24 processed frames; history is replay-only. Source frame
300 at 10.00 seconds belongs to all three inputs and shows three framed
works beside an electronic display (Figure~\ref{fig:case-all-view-unresolved}).
The recorded proposals cannot recover the reference, but the counting scope
remains ambiguous. This is not an established representation failure, nor
can we determine whether the model included the display in its count.

\begin{figure}[H]
\centering
\includegraphics[height=0.72\textheight,width=\textwidth,keepaspectratio]{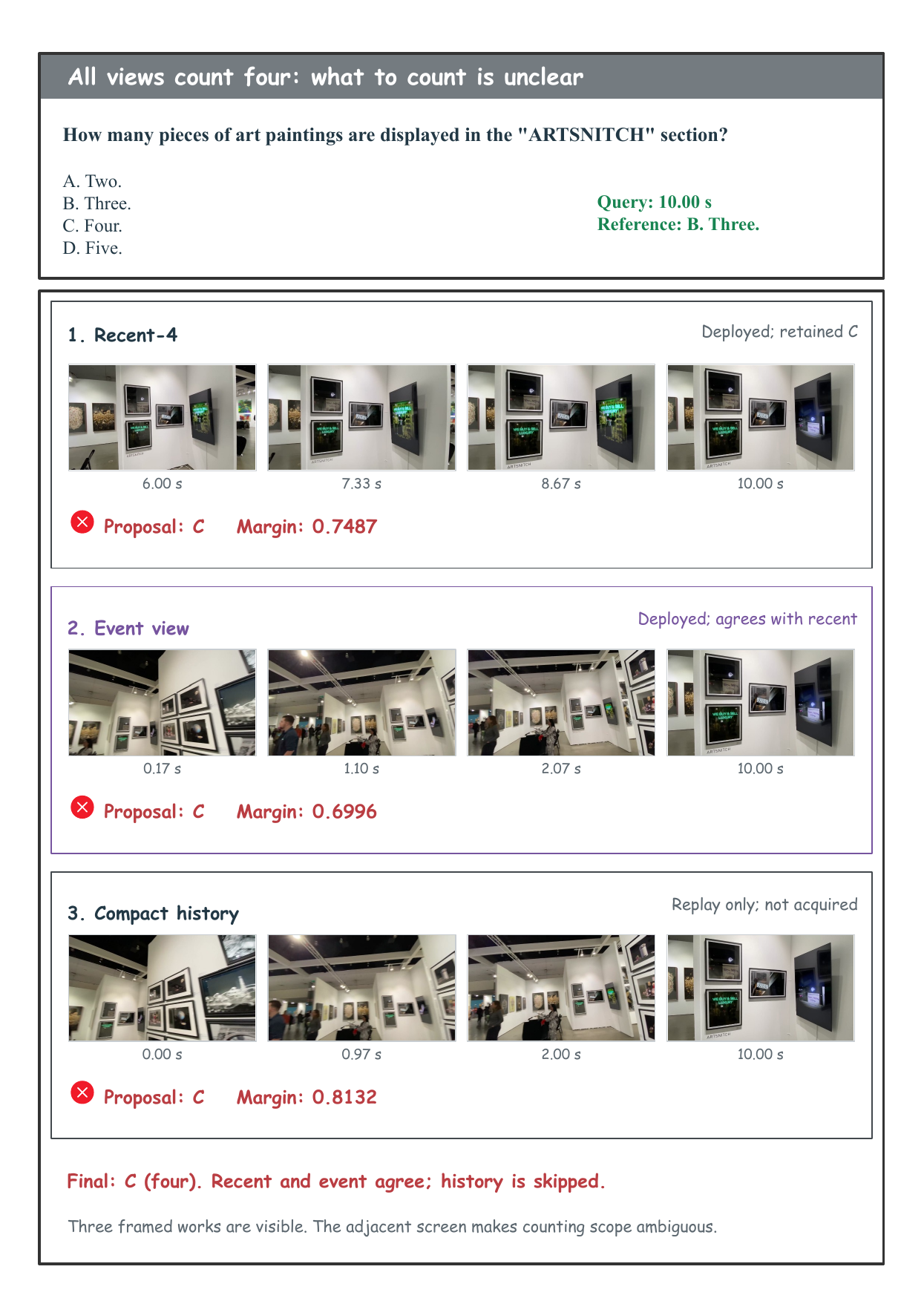}
\caption{\textbf{All-view failure with ambiguity about what to count.}
Four actual frames per view are displayed; complete inputs contain 4 recent,
20 event, and 21 history frames. The same frozen \viss{} checkpoint is used
throughout. Exact traces for both cases are provided in
\texttt{supplementary/viss\_failure\_cases.json}; no new proposals are generated.}
\label{fig:case-all-view-unresolved}
\end{figure}

\clearpage
\subsection{Additional Archived-Run Comparisons}
\label{app:additional-qualitative}

Figures~\ref{fig:case-movie-inserts} and~\ref{fig:case-next-formula} present
selected archived-run comparisons. The two controls share a checkpoint;
ours uses the ViSS checkpoint. Checkpoints and, for expanded views, visual
encodings differ, so these are qualitative examples rather than isolated
ablations. Rows show actual input frames, with subsets displayed for longer
views; all frames remain within the observed prefix.

\begin{figure}[H]
\centering
\includegraphics[height=0.73\textheight,width=0.95\textwidth,keepaspectratio]{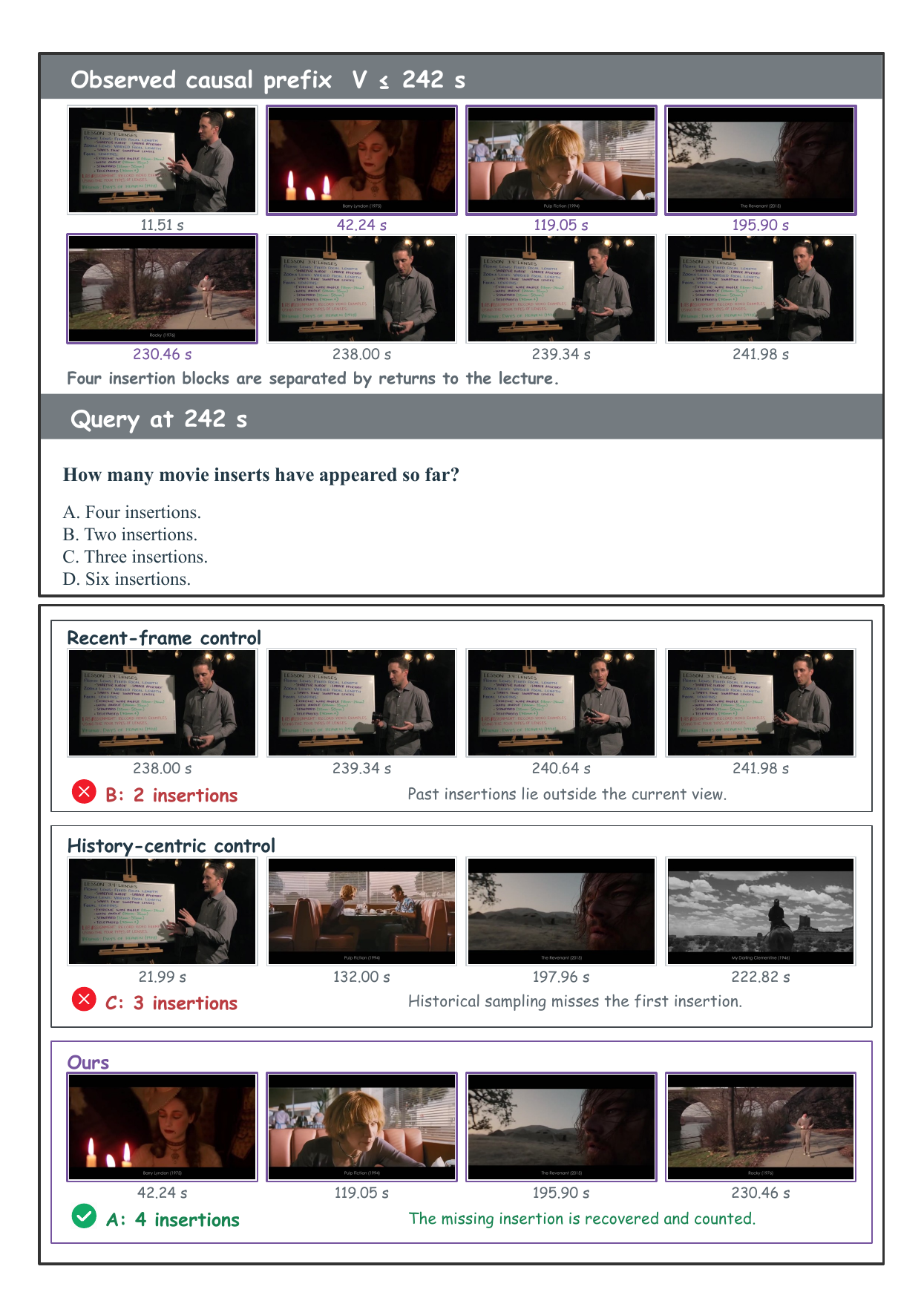}
\vspace{-0.3em}
\caption{\textbf{Accumulating movie insertions.} At 242 seconds, the recent-frame
control predicts two insertions and the historical-search control predicts
three. Historical sampling omits the first insertion block. Our causal-access
run exposes this missing evidence and returns four, matching the reference.
Each insertion block may contain several consecutive film excerpts.}
\label{fig:case-movie-inserts}
\end{figure}

\clearpage
\begin{figure}[H]
\centering
\includegraphics[height=0.81\textheight,width=0.95\textwidth,keepaspectratio]{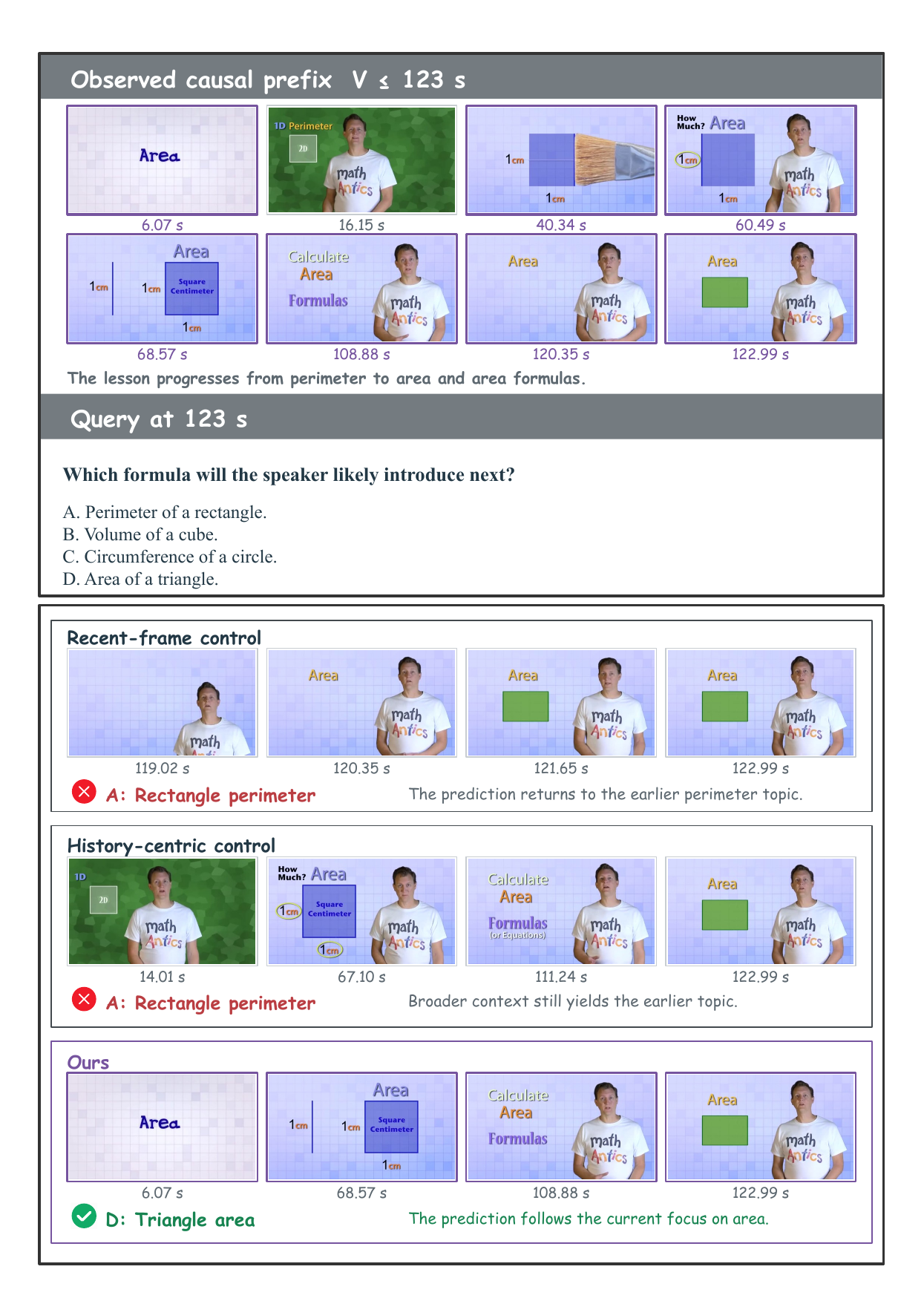}
\vspace{-0.3em}
\caption{\textbf{Following the current lesson topic.} The observed lesson moves
from perimeter to area and area formulas. Both controls predict rectangle
perimeter (A), while our run predicts the annotated triangle-area option (D).
The displayed prefix establishes the current area topic; no future frames
are exposed to the model.}
\label{fig:case-next-formula}
\end{figure}

\clearpage
\section{Theoretical Analysis}
\label{app:theory}

This appendix makes the two levers precise. Proposition~\ref{prop:underspec}
formalizes state under-specification (\cone{}) and shows what a separating
probe adds to an answer-level objective. Proposition~\ref{prop:monotone}
shows that for a Bayes-optimal readout a wider causal view never loses
accuracy and no commitment rule can beat it. The replacement errors in
Section~\ref{sec:results-controllers} concern the deployed readout and
selection policy within a fixed proposal set. Propositions~\ref{prop:threshold}
and~\ref{prop:commit} give conditions under which the margin gate of \tvs{}
is the optimal acquisition rule and explain why its calibrated commit
threshold is negative. All statements concern population quantities and
unrestricted predictor classes: they describe what the objectives and
decision rules can and cannot enforce, not finite-sample behavior.

\subsection{Setting}
\label{app:theory-setting}

Let $\mathcal H$ be the set of causal prefixes (histories), $\mathcal Q$ the
set of questions, and $a:\mathcal Q\times\mathcal H\to\mathcal Y$ the
reference answer map into a finite answer set. A predictor factors as in
Eq.~\ref{eq:representation-readout}: $f(q,h)=G(F(g(h)),q)$, where $g$ is a
fixed view constructor, $F$ maps a view to a representation, and $G$ is a
readout returning a distribution over $\mathcal Y$; $G(z,q)[y]$ denotes the
probability assigned to $y$. A supervision family is a finitely supported
distribution $D$ over pairs $(q,h)$, and its population risk under log loss is
\begin{equation}
R_D(F,G)=\mathbb E_{(q,h)\sim D}\bigl[-\log G(F(g(h)),q)[a(q,h)]\bigr].
\label{eq:theory-risk}
\end{equation}
$F$ and $G$ range over all measurable maps. Call $D$ \emph{view-consistent}
if $a(q,h)=a(q,h')$ whenever $g(h)=g(h')$ and both $(q,h)$ and $(q,h')$
belong to its support. Under this condition, $\inf R_D=0$ is attained by
an injective $F$ and a readout assigning the reference answer to each view
and question. Without it, information lost by $g$ can leave positive Bayes risk.
For a view $x$, define its question--answer support across the entire class
$g^{-1}(x)$ by
$\mathcal L_D(x)=\{(q,a(q,h)):(q,h)\in\operatorname{supp}(D),\ g(h)=x\}$.
We say that $D$ \emph{does not separate} two view classes when their
$\mathcal L_D$ sets are equal.
For example, an action-recognition family can leave once-versus-three
repetitions indistinguishable; a counting family can separate them.

\subsection{State Under-Specification}

\begin{proposition}[Answer-level supervision leaves unasked distinctions free]
\label{prop:underspec}
Let $h_1\neq h_2$ have distinct views $x_1=g(h_1)$ and $x_2=g(h_2)$.
Suppose $D$ is view-consistent and
$\mathcal L_D(x_1)=\mathcal L_D(x_2)$.
\begin{enumerate}[label=(\roman*),nosep,leftmargin=*]
\item For every minimizer $(F,G)$ of $R_D$, the merged pair $(F',G)$ with
$F'(g(h_2))=F(g(h_1))$ and $F'=F$ elsewhere is also a minimizer. The
objective therefore never requires the representation to keep any attribute
in which $h_1$ and $h_2$ differ, be it a count, an order, or a relation.
\item Let $D_p$ be a probe family with a question $q_p$ such that
$a(q_p,h_1)\neq a(q_p,h_2)$ and
$w=\min\{D_p(q_p,h_1),\allowbreak D_p(q_p,h_2)\}>0$, and
let $D'=\lambda D+(1-\lambda)D_p$ with $\lambda\in[0,1)$, and assume $D'$
is view-consistent. Every $(F,G)$ with
$F(g(h_1))=F(g(h_2))$ satisfies $R_{D'}(F,G)\geq 2(1-\lambda)\,w\log 2$,
while $\inf R_{D'}=0$. Hence every minimizer of $R_{D'}$, and every
$\varepsilon$-approximate minimizer with $\varepsilon<2(1-\lambda)w\log 2$,
keeps $F(g(h_1))\neq F(g(h_2))$.
\end{enumerate}
\end{proposition}

\begin{proof}
(i) Since $\inf R_D=0$ and every term of Eq.~\ref{eq:theory-risk} is
nonnegative, a minimizer places probability one on $a(q,h)$ at every support
point. The merged pair agrees with $(F,G)$ at every support point with
$g(h)\neq x_2$. For any support point $(q,h)$ with $g(h)=x_2$,
$(q,a(q,h))\in\mathcal L_D(x_2)=\mathcal L_D(x_1)$, so a support point
$(q,\bar h)$ with $g(\bar h)=x_1$ has the same target. Hence
$G(F'(g(h)),q)=G(F(x_1),q)$ assigns probability one to $a(q,h)$. Thus
$R_D(F',G)=0$.
(ii) If $F(g(h_1))=F(g(h_2))=z$, the readout returns one distribution
$p=G(z,q_p)$ at both $(q_p,h_1)$ and $(q_p,h_2)$, whose reference answers
$y_1\neq y_2$ are distinct. These two support points contribute at least
$(1-\lambda)\,w\,(-\log p[y_1]-\log p[y_2])$ to $R_{D'}$, and
$p[y_1]+p[y_2]\leq 1$ gives $p[y_1]\,p[y_2]\leq 1/4$, so the bracket is at
least $2\log 2$. All other terms are nonnegative. View-consistency of $D'$
ensures that an injective $F$ on its support views and a suitable $G$ attain
the infimum $0$.
\end{proof}

\begin{remark}[The view can hide the distinction]
If $g(h_1)=g(h_2)$, the two histories are the same input to every predictor
and no supervision separates them; the distinction must be restored by the
view constructor. This is \ctwo{}, and it is why \tvs{} changes $g$ while
$\theta^\star$ stays frozen. The semantic-stability check of
Section~\ref{sec:method-viss} tests invariance to option permutation; it does
not establish the view-consistency assumption or human unambiguity of a target.
\end{remark}

Part (i) formalizes \cone{} for view classes with identical supervised
question--answer support: the objective permits merging their representations.
Part (ii), under view-consistency, is the design rule behind the probe families of
Section~\ref{sec:method-viss}: one separating question per distinction, with
positive weight, is enough to make every minimizer keep the distinction in
$z$. Whether benchmark readouts then use it is an empirical question, which
Section~\ref{sec:results-data} answers at a matched budget.

\subsection{Wider Views and Commitment}

We now fix the checkpoint and vary the view, as \tvs{} does. Let $(q,h)$ be
drawn from a query distribution, let $x=g_0(h)$ be the recent view and
$x'=g_1(h)$ a wider view that contains it, in the sense that $x=\phi(x')$ for
a fixed map $\phi$ (the event view includes the four recent frames). Because
a view discards part of the prefix, the reference answer is not a function
of the view, and we write $P(y\mid x,q)$ for its conditional law. A
\emph{readout} on a view is any measurable decision $\delta(x,q)\in\mathcal Y$;
its accuracy is $A(\delta)=P(\delta(x,q)=a(q,h))$, and the Bayes readout
$\delta^\star_x(x,q)=\argmax_y P(y\mid x,q)$ attains the Bayes accuracy
$A^\star(x)=\mathbb E[\max_y P(y\mid x,q)]$.

\begin{proposition}[A wider view never loses for a Bayes readout]
\label{prop:monotone}
If $x=\phi(x')$, then $A^\star(x')\geq A^\star(x)$. Moreover, let
$\delta_0$ and $\delta_1$ be any fixed readouts on $x$ and $x'$ with margins
$m_0(x,q)$ and $m_1(x',q)$, and let $\rho$ be any commitment rule that
returns one of the two proposals as a function of
$(\delta_0,\delta_1,m_0,m_1)$. Then $A(\rho)\leq A^\star(x')$, and if
$\delta_1=\delta^\star_{x'}$ then $A(\rho)\leq A(\delta_1)$: no commitment
rule improves on always taking the wider view's proposal.
\end{proposition}

\begin{proof}
Every decision that depends on $x$ depends on $x'$ through $\phi$, so it is
a measurable function of $(x',q)$; the Bayes readout on $x'$ maximizes
accuracy over all such functions, which gives the first claim with
$\delta=\delta^\star_x\circ\phi$. For the second, the proposals and margins
are deterministic functions of $(x',q)$ once the readouts are fixed, so
$\rho$ is itself a measurable decision on $(x',q)$ and
$A(\rho)\leq A^\star(x')$; when $\delta_1$ is the Bayes readout the
right-hand side equals $A(\delta_1)$.
\end{proof}

Proposition~\ref{prop:monotone} separates two readings of the replacement
errors in Section~\ref{sec:results-controllers}. For a Bayes readout on the
wider view, replacing the retained answer can still be wrong on individual
queries, but no rule that keeps the retained answer more often can raise
accuracy. The deployed readout need not be Bayes-optimal, so this result
does not imply that proposal selection cannot improve its accuracy. The
three-view oracle uses reference answers and is not a Bayes-risk estimate;
its gap diagnoses selection failures within the current proposal set.
The commit-margin ablation in Section~\ref{sec:results-data} evaluates one
particular selection rule, not the potential of all commitment rules.

\subsection{The Acquisition Gate and the Commit Threshold}

Let $M=m_0(x,q)$ be the recent-view margin, let
$p_0(m)=P(\delta_0\ \text{correct}\mid M=m)$, and let
$p_1(m)=P(\text{final answer correct}\mid M=m,\ \text{acquire})$ under the
fixed downstream rule of \tvs{}. Let $c(m)$ be the expected number of
additional frames processed when acquiring at margin $m$, at most the event
and history budgets $c_1+c_2$. A gate is a measurable set
$A\subseteq\mathbb R$ of margins at which the event view is acquired, and its
value at frame price $\lambda\geq 0$ is
$J(A)=P(\text{correct})-\lambda\,\mathbb E[\text{frames}]$.

\begin{proposition}[The margin gate is the optimal acquisition rule]
\label{prop:threshold}
Write $\Delta(m)=p_1(m)-p_0(m)$ for the gain from acquiring at margin $m$.
The gate maximizing $J$ is $A^\star=\{m:\Delta(m)>\lambda\,c(m)\}$. If
$\Delta(m)-\lambda c(m)$ is non-increasing in $m$ (for instance, $\Delta$
non-increasing and $c$ constant), then $A^\star=\{m<\tau_a\}$ with
$\tau_a=\sup\{m:\Delta(m)>\lambda c(m)\}$: the optimal rule acquires exactly
when the recent margin falls below a threshold, as \tvs{} does, and its
expected frame count is
$c_0+P(M<\tau_a)\bigl[c_1+P(\text{disagree}\mid M<\tau_a)\,c_2\bigr]$ with
$c_0,c_1,c_2$ the per-view frame budgets.
\end{proposition}

\begin{proof}
$J(A)=\mathbb E[p_0(M)]-\lambda c_0+\mathbb E\bigl[\mathbf 1\{M\in A\}
\,(\Delta(M)-\lambda c(M))\bigr]$, and the last expectation is maximized
pointwise by including $m$ exactly when $\Delta(m)>\lambda c(m)$. If
$\Delta-\lambda c$ is non-increasing, this set is a lower set of the real
line, hence $\{m<\tau_a\}$ up to its boundary point. The frame count follows
from the cascade of Section~\ref{sec:method-agent}: the history view is
acquired only after the event view and only on disagreement.
\end{proof}

The monotonicity of $\Delta$ says that additional evidence is worth less
when the current answer is already confident; it is the assumption behind
confidence-gated computation
\citep{kadavath2022know,kamath2020selective,snell2025scaling}, and it can be
checked on the calibration holdout by binning $M$. Without a frame price
($\lambda=0$) the optimal gate acquires whenever $\Delta>0$, so a finite
$\tau_a$ expresses a positive price on frames rather than a belief that the
wider view could hurt.

\begin{proposition}[A negative commit threshold is a calibration offset]
\label{prop:commit}
Suppose disagreeing proposals $j$ (retained) and $k$ (new) have correctness
probabilities that depend on the margins only through view-specific
increasing calibration maps,
$P(j\ \text{correct}\mid m_j,m_k)=\varphi_j(m_j)$ and
$P(k\ \text{correct}\mid m_j,m_k)=\varphi_k(m_k)$. The accuracy-maximizing
commitment rule takes $k$ iff $\varphi_k(m_k)>\varphi_j(m_j)$. If the two
maps differ by a shift, $\varphi_k(m)=\varphi_j(m+\delta)$, this rule reads
$m_k-m_j>-\delta$, which is the \tvs{} rule with $\tau_c=-\delta$.
\end{proposition}

\begin{proof}
Given the margins, keeping $j$ is correct with probability $\varphi_j(m_j)$
and taking $k$ with probability $\varphi_k(m_k)$, so picking the larger
probability maximizes expected accuracy pointwise. Under the shift,
$\varphi_k(m_k)>\varphi_j(m_j)$ iff $\varphi_j(m_k+\delta)>\varphi_j(m_j)$
iff $m_k+\delta>m_j$, because $\varphi_j$ is increasing.
\end{proof}

A negative calibrated $\tau_c$ (Section~\ref{sec:method-agent}) is thus what
one expects when wider views yield lower raw margins at equal accuracy, as
happens when more frames spread probability over more plausible answers; the
holdout calibration estimates $\delta$ without modeling $\varphi_j$
explicitly. Proposition~\ref{prop:monotone} adds the caveat that, for a Bayes
readout on the wider view, the best value of this rule is to always accept,
so any accuracy the rule adds in practice is a symptom of readout
miscalibration that it partially corrects.

\end{document}